\documentclass{article}
\usepackage[utf8]{inputenc} 
\usepackage[T1]{fontenc}    
\usepackage{arxiv}
\usepackage{cite}
\usepackage{bigstrut}
\usepackage{multirow}
\usepackage{amsmath,amssymb,amsfonts}
\usepackage{graphicx}
\usepackage{tabto}
\usepackage{textcomp}
\DeclareMathOperator*{\argmin}{argmin}
\usepackage{amsthm}
\usepackage{amsfonts}
\newtheorem{theorem}{Theorem}[section]
\newtheorem{conjecture}{Conjecture}[section]
\newtheorem{definition}{Definition}[section]
\usepackage{mathtools}

\usepackage[linesnumbered,ruled,vlined]{algorithm2e}
\def\BibTeX{{\rm B\kern-.05em{\sc i\kern-.025em b}\kern-.08em
    T\kern-.1667em\lower.7ex\hbox{E}\kern-.125emX}}

\begin{document}
\title{Stochastic Configuration Machines for Industrial Artificial Intelligence}
\author{Dianhui Wang
\thanks{Corresponding Author (dh.wang@deepscn.com)}\\
Artificial Intelligence Research Institute, China University of Mining and Technology\\
Xuzhou 221116, China \\
State Key Laboratory of Synthetical Automation for Process Industries\\
Northeastern University, Shenyang 110819, China\\
\And
Matthew J. Felicetti\\
Department of Engineering, La Trobe University\\
Bendigo VIC 3552, Australia}

\maketitle

\begin{abstract}
Real-time predictive modelling with desired accuracy is highly expected in industrial artificial intelligence (IAI), where neural networks play a key role. Neural networks in IAI require powerful, high-performance computing devices to operate a large number of floating point data. Based on stochastic configuration networks (SCNs), this paper proposes a new randomized learner model, termed stochastic configuration machines (SCMs), to stress effective modelling and data size saving that are useful and valuable for industrial applications. Compared to SCNs and random vector functional-link (RVFL) nets with binarized implementation, the model storage of SCMs can be significantly compressed while retaining favourable prediction performance. Besides the architecture of the SCM learner model and its learning algorithm, as an important part of this contribution, we also provide a theoretical basis on the learning capacity of SCMs by analysing the model’s complexity. Experimental studies are carried out over some benchmark datasets and three industrial applications. The results demonstrate that SCM has great potential for dealing with industrial data analytics.
\end{abstract}


\section{Introduction}
Industrial artificial intelligence (IAI) stresses the application of artificial intelligence techniques to industries, with some inherent challenges, such as uncertainties in sensory signals, real-time data processing, high modelling accuracy, and the interpretability of predictive models and results \cite{120301, lee2018industrial, 9285283, LEE201820, DWIVEDI2021101994, abiodun2018state,doi:10.1080/21693277.2016.1192517}.
Recently, the IAI concept has received considerable attention worldwide due to the availability
of cheaper sensors for data acquisition, powerful computing facilities and advanced algorithms that
perform speedily at lower computational cost, larger storage devices and cloud computing
technology for data management, and faster communication systems for sharing and delivering data.
Although the IAI concept is not well-defined so far, the development of advanced machine learning
algorithms is strongly expected so that they can meet these requirements of IAI.

Machine learning has been a very active research area in AI over the past decades, and significant efforts in building predictive learner models have been made \cite{su12020492}. Among these approaches, the most
popular and widely used ones include multilayer perceptrons with
error-backpropagation algorithms (MLPs) \cite{10.5555/104279.104293}, support vector machines (SVMs) \cite{SVM}, Bayesian
networks (BNs) \cite{friedman1997bayesian}, and adaptive neuro-fuzzy inference systems (ANFIS) \cite{256541}. In recent years, deep neural networks (DNNs) with different learning strategies are dominant and ubiquitous due to
their excellent performance for problem-solving in some domains such as computer vision  \cite{yoo2015deep},
natural language processing \cite{9075398},  medical diagnosis \cite{yadav2019deep, chen2017deep}, speech recognition \cite{abdel2014convolutional} and financial analysis \cite{zhang2021application}.
Unfortunately, these aforementioned methods can hardly perform in a wide range of IAI tasks because of the constraints on data quality, real-time
requirement, accuracy expectation, and interpretability of results. Indeed,
except for these concerns with learner models and learning algorithms, issues related to hardware implementations of large learner models must be considered for industrial applications where data processing power is crucial with limited-resource devices \cite{DBLP:journals/corr/WuLWHC15, DBLP:journals/corr/ZhangZLS17, wu2016quantized, chen2019eyeriss}. 

Much research has been done into reducing data by using feature extraction and pruning methods. However, another approach to lower the computational burden is to reduce the amount and accuracy of the physical data. Typically, neural network weights are stored as either 32-bit or 64-bit floating-point values. Gupta and Narayanan \cite{DBLP:journals/corr/GuptaAGN15} demonstrate that computational improvements can be achieved with good generalization using limited numerical precision, namely 16-bit fixed-point values. Further, activation functions require division or many multiplications, which can take up to 10-20 processor cycles per operation. Courbariaux, David and Bengio \cite{courbariaux2014training} show that low-precision multiplications are sufficient for training neural networks. Courbariaux, David and Bengio  \cite{DBLP:journals/corr/CourbariauxBD15} then demonstrate the use of binary weights during forward and back propagation, known as BinaryConnect, which in turn can be used for memory reduction and speed enhancement. BinaryConnect presents two binarization techniques of the weights $w$, a deterministic approach and stochastic based on the sign function. The BinaryConnect model is trained using Stochastic Gradient Descent (SGD), with two sets of weights, real and binarized. The real weights are binarized into binary weights to be passed through the forward and backward pass, and the real weights are updated after each pass. Due to the real weights extending past the range of binary values, they are bound between -1 and 1 using a clipping function. Extending this, Courbariaux et al. demonstrate that this approach could be extended to binary activations in what is known as the Binarized Neural Network \cite{Courbariaux2016BinarizedNN}. The binary activation function sign is generally unusable with back propagation due to the derivatives resulting in 0 in most cases; hence, a technique known as straight-through estimator \cite{DBLP:journals/corr/BengioLC13} is used to approximate the derivative of the sign function.  
Further, binary weights, activation and even inputs can also have the advantage of using XNOR operations for fast bit-level multiplications rather than costly floating point multiplication \cite{ Courbariaux2016BinarizedNN,DBLP:journals/corr/KimS16, DBLP:journals/corr/RastegariORF16}. Due to the binarization it is expected that there is a loss of information in the network, hence Rastegari et.al. have presented an approach to reduce the binarization (quantization) error by apply a scaling factor to the binarized weights known as XNOR-Net\cite{DBLP:journals/corr/RastegariORF16}. 
Recently in 2019, Ding et. al. expressed that training binarized neural networks can be difficult due to degeneration, saturation and problems with gradient  mismatch\cite{DBLP:journals/corr/abs-1904-02823}. And provides a solution to this using regularization by using the distribution loss where the distribution loss is the loss caused by degeneration, saturation and gradient mismatch. Another area of interest, to still reduce the amount of information stored but also reduce the quantization error, is ternary weights \cite{DBLP:journals/corr/LiL16, DBLP:journals/corr/AlemdarCLPP16}. Zhang and Liu\cite{DBLP:journals/corr/LiL16} employ a threshold-based ternary function to train ternary-valued weights which shows the advantage both in compression compared to real weights and that this method outperforms binary weights in terms of performance. Further from this, to push neural networks faster with lower power consumption and to produce networks with a smaller physical size we see many examples of hardware implementations of neural networks \cite{ZHANG2020106, 6402898, 7551399, 8456540}.

Built on our proposed Stochastic Configuration Networks (SCNs) concept \cite{8013920}, this paper aims at developing a new randomised learner model, termed Stochastic Configuration Machine (SCM) for IAI applications. Three remarkable characterizations of SCM can be summarized as follows:
\begin{itemize}
  \item SCM model is composed of a mechanism model and a DeepSCN model \cite{DBLP:journals/corr/WangL17c}.
  \item  The input variables of the linear part (i,e, the direct link from the input layer to the output layer) are selective.
  \item The random weights and biases take binary values and real values, respectively.
\end{itemize}
Through comprehensive comparisons, it is found that SCM has merits in reducing memory load and faster training whilst still producing good and reliable performance. Notice that the mechanism model used in SCM can be replaced by a simulation or fuzzy expert system. This part has two functions, that is, reducing the complexity of modelling tasks and seeking a solution of interpretable results from SCM. Suppose that the mechanism model has no requirement on tuning its parameters during the learning process, we can simply make a difference between the targets and the outputs from the mechanism model to form a new target for modelling. We deal with this case in this paper.

The remainder of the paper is organised as follows: Section 2 reviews the randomised learner models, including SCN and RVFL networks. Section 3 details our proposed framework, including the description of SCM models, a definition of the model's complexity, two theoretical results on the one-order universal approximation property of SCM based on the complexity concept and a detailed learning algorithm with pseudo code. Section 4 reports our experimental results over some benchmark datasets and three industrial datasets, with comparisons, discussion and hardware implementation. Section 5 concludes the paper. 

\section{Related Work}
\subsection{Random Vector Functional-Link Networks}
Randomized learners have been introduced to address the issue of efficient training of neural networks by assigning the weights and biases randomly. It has been demonstrated that randomized learners can improve the speed of building neural networks, create simpler algorithms and lower computational cost \cite{LI2017170, mahoney2011randomized, doi:10.1002/widm.1200, editorial_randomized, broomhead1988radial}. Random vector functional-link (RVFL) networks \cite{Pao} are a class of feed-forward neural networks with a direct link between the input layer and the output layer, in which the weights and biases in the hidden layer are chosen randomly from the uniform distribution over a given range, and the output weights are evaluated by using the least squares methods. It has been shown \cite{471375} that RVFL networks can be universal approximators, however, it was also shown \cite{Husmeier1999RandomVF} that RVFL fails to work if the random parameters for the weights and biases are set improperly. 
Indeed, these theoretical results established in \cite{471375} do not really help in designing RVFL networks for problem-solving \cite{gorban2016approximation, tyukin2009feasibility}. As a matter of fact, the scope setting of random weights and biases directly impacts the modelling performance \cite{DUDEK201933}. 
An incremental approach to building the RVFL networks is known as incremental RVFL (IRVFL), and this implementation ensures a unified framework for comparison as outlined in Section \ref{s_c_m}. One may refer to \cite{doi:10.1002/widm.1200} for more information on the randomized learning techniques.



\subsection{Stochastic Configuration Networks} \label{s_c_n}
Let us consider a continuous target function $f: \mathbb{R}^d\to\mathbb{R}^m$, and a given SCN model with $L-1$ hidden nodes, $f_{L-1}(x)=\sum^{L-1}_{j=1}\beta_j\phi_j(w_j^Tx+b_j)$,$(L=1,2,..., f_0=0)$, where $\beta_j=[\beta_{j,1},...,\beta_{j,m}]^T$. The current residual error is given by $e_{L-1}=f-f_{L-1}=[e_{L-1,1},...,e_{L-1,m}]^T$. If $||e_{L-1}||$ does not reach a pre-defined tolerance level, then a new random basis function $\phi_L(w_L$ and $b_L)$ is generated, and the output weights $\beta_L$ are evaluated so that the leading model $f_L=f_{L-1}+\beta_L\phi_L$ will have an improved residual error. Mathematically, it can be shown that SCN models built constructively using the least squares method for updating the output weights based on the stochastic configuration learning algorithms \cite{8013920} are universal approximators. \\
\begin{theorem}
Let $\Gamma$ be a set of basis functions. Suppose that $span(\Gamma)$ is dense in $L_2$ functional space, and for $\forall \phi\in\Gamma$,$0<||\phi||<b_\phi$ for some $b_\phi\in\mathbb{R}^+$. Given $0<r<1$ and a non-negative real number sequence $\{\mu_L\}$ with $\lim_{L\to+\infty}\mu_L=0$ and $\mu_L\le(1-r)$. Denoted by
\begin{equation}
    \delta_L=\sum_{q=1}^m\delta_{L,q},\delta_{L,q}=(1-r-\mu_L)||e_{L-1,q}||^2, L=1,2,...
\end{equation}
If the random basis function $\phi_L$ is generated to satisfy the following inequalities:
\begin{equation}
    \left<e_{L-1,q},\phi_L\right>^2\ge b^2_\phi\delta_{L,q}, q=1,2,...,m,
\end{equation}
and the output weights are evaluated by
\begin{equation}
    [\beta_1^*,\beta_2^*,...,\beta_L^*]=\argmin_\beta||f-\sum^L_{j=1}\beta_j\phi_j||,
\end{equation}
then $\lim_{L\to+\infty}||f-f_L^*||=0$, where $f_L^*=\sum^L_{j=1}\beta_j^*\phi_j,\beta_j^*=[\beta_{j,1}^*,...,\beta_{j,m}^*]^T$.
\end{theorem}
To make the paper self-contained and complete, we introduce some notation followed by the construction steps of SCN as follows. Given a training data set with N sample pairs $\{(X_p,Y_p),p=1,2,...,N\}$, where $X_p=[x_1,x_2,...,x_d]^T\in\mathbb{R}^d$ and  $Y_p=[y_1,y_2,...,y_m]^T\in\mathbb{R}^m$. Let $X\in\mathbb{R}^{N\times d}$ and $Y\in\mathbb{R}^{N\times m}$ represent the input and output data matrix, respectively; $e_{L-1}(X)\in\mathbb{R}^{N\times m}$ be the residual error matrix for the SCN model with $L-1$ terms, where each column $e_{L-1,q}(X)=[e_{L-1,q}(X_1),...,e_{L-1,q}(X_N)]^T\in\mathbb{R}^N, q=1,2,...,m$, Denoted the output vector of the \emph{L-th} hidden node $\phi_L$ for the input data matrix $X$ by
\begin{equation}
    h_L(X)=[\phi_L(w_L^TX_1+b_L),...,\phi_L(w_L^TX_N+b_L)]^T.
\end{equation}
Thus, the hidden layer output matrix of the SCN model can be expressed as $H_L=[h_1,h_2,...,h_L]$. Donated by
\begin{equation}
    \xi_{L,q}=\frac{\left<e^T_{L-1,q}(X),h_L(X)\right>^2}{\left<h^T_L(X),h_L(X)\right>}-(1-r_L)\left<e_{L-1,q}^T(X),e_{L-1,q}(X)\right>, q=1,2,...,m.
\end{equation}
With these notations, the baseline stochastic configuration algorithm (SCA) can be outlined as follows:
\begin{description}
    \item[Step 1.]  Set up the learning parameters, including a set of positive scalars $\lambda_i\in[\lambda_{min}, \lambda_{max}], i =1,2,...,s$, where $\lambda_{min}$ and $\lambda_{max}$ are two predefined values, and an increasing sequence $r_1< r_2<...<r_t < 1$; Also, we set up two termination conditions, that is, either the maximum number of the hidden nodes $L_{max}$ or the error tolerance $\tau$.
    \item[Step 2.]  Take random parameters $w_L$ and $b_L$ from adjustable interval $[-\lambda, \lambda]$ for a user-specified number of times, and check the following inequalities with $r_i, i=1,2,...,t$ (starting from $r_1$)
    \begin{equation} \label{eq:6}
        \xi_{L,q}\ge 0,q=1,2,...,m.
    \end{equation}
    If (\ref{eq:6}) holds, define the set of random parameters $w_L$ and $b_L$ such that $\xi_L=\sum_{q=1}^m\xi_{L,q}$ takes the maximum.
    \item[Step 3.]  Evaluate the output weight matrix $\beta$ by solving the following least means square problem:
     \begin{equation}
        \beta^*=\argmin_\beta||H_L\beta-Y||^2_F=H_L^+Y,
    \end{equation}
    where $H_L^+$ is the Moore-Penrose generalized inverse of the matrix $H_L$, and $||\cdot||_F$ represents the Frobenius norm.
\end{description}
For more details on SCN and DeepSCN fundamentals and algorithms, we recommend readers to refer to \cite{8013920, DBLP:journals/corr/WangL17c}.
\section{Stochastic Configuration Machines} \label{s_c_m}
This section details our proposed SCM model, associated learning algorithm and learner’s capacity for approximating/representing nonlinear signals. As a matter of fact, SCM is a generalized DeepSCN model with some specific settings, that is, purposely adding a mechanism model for cognitive modelling to a DeepSCN model where random weights take binary values. A visual representation is provided in Figure \ref{fig:SCM_model}. The following concepts and notation will be used in this paper. Let $\Gamma:=\{H_1,H_2,H_3,...\}$ be a set of real-l-valued functions, span($\Gamma$) denote a function space spanned by $\Gamma$; $L_2(D)$ denote
the space of all continuous functions $f=[f_1,f_2,...,f_m] : \mathbb{R}^d\to\mathbb{R}^m$ defined on $D\subset\mathbb{R}^d$, with the $L_2$ norm defined as
\begin{equation}
\|f\|:=\left(\displaystyle\sum_{q=1}^{m} \int_D|f_q(x)|^2dx \right)^{1/2}<\infty
\end{equation}
The inner product of $\phi=[\phi_1,\phi_2,...,\phi_m]:\mathbb{R}^d\to\mathbb{R}^m$ and f is defined as
\begin{equation}
\langle f,\phi \rangle:=\displaystyle\sum_{q=1}^{m}\langle f_q,\phi_q \rangle=\displaystyle\sum_{q=1}^{m}\int_D f_q(x)\phi_q(x)dx 
\end{equation}
In the special case that $m=1$, for a real-valued function $\psi:\mathbb{R}^d\to\mathbb{R}$ defined on $D\subset\mathbb{R}^d$, its $L_2$ norm becomes $\|\psi\|:=(\int_D|\psi(x)|^2dx)^{(1/2)}$, while the inner product of $\psi_1$ and $\psi_2$ becomes $\langle \psi_1,\psi_2 \rangle=\int_D\psi_1(x)\psi_2(x)dx$.
\subsection{SCM Model}
\begin{equation}
Y=P(X,p,u)+\displaystyle\sum_{k=1}^{M}\beta_k H_k(X),
\end{equation}
where $P(X,p,u)=P_0(X,p,u)+L(\bar{X}),P_0(X,p,u)$ represents a mechanism model with a set of parameters $p=[p_1,...,p_l]^T(l\le d)$ and control inputs $u=[u_1,...,u_m]$, $L(\bar{X})$ is a linear regression model with selective input variables $\bar{X}(\bar{X}\subseteq X)$, and
\begin{equation}
H_k(X)=\phi(W_k^T H_{k-1}(X)+\Theta_k),
\end{equation}
where $\phi(\cdot)$ is an activation function (no limit on its differentiability), $\Theta_k=[\theta_1^k,...,\theta_{n_k}^k]^T$ denotes a threshold vector of hidden nodes at the \emph{k-th} layer, $H_0(X)=X$, $k=1,2,..., M$, $n_M=m$ and
\begin{equation}
W_{k} = 
 \begin{bmatrix}
  w_{1,1}^k &  \cdots & w_{1,n_k}^k \\
  \vdots  & \vdots  & \vdots  \\
  w_{n_{k-1},1}^k & \cdots & w_{n_{k-1},n_k}^k 
 \end{bmatrix}_{n_{k-1}\times n_k}
\end{equation}
where $w_{ij}^k$ takes binary values $\{-\lambda,+\lambda\}$, $\lambda\in\{\lambda_1,\lambda_2,...,\lambda_p\}$, representing a random synaptic weight between the \emph{i-th} node at the \emph{k-th} layer (the \emph{0-th} layer is the input layer) and the \emph{j-th} node at the \emph{(k+1)-th} layer (the \emph{L-th} layer is the output layer); $\beta_k$ denotes the readout (or output weight) matrix from the \emph{k-th} hidden layer to the output layer:
\begin{equation}
\beta_{k} = 
 \begin{bmatrix}
  \beta_{1,1}^k &  \cdots & \beta_{1,n_k}^k \\
  \vdots  & \vdots  & \vdots  \\
  \beta_{m,1}^k & \cdots & \beta_{m,n_k}^k 
 \end{bmatrix}_{m\times n_k},
 \beta=\begin{bmatrix}
\beta_1^T \\
\beta_2^T  \\
 \vdots \\
\beta_m^T  
 \end{bmatrix}.
\end{equation}\\
\emph{\textbf{Remark 1:}} If no mechanism model is available, we can simply replace the $P(X,p,u)$ by a linear regression model. It is important to notice that $P(X,p,u)$ can be a simulation or fuzzy expert system, which will play a key role in cognitive learning and interpretable AI for industrial informatics.
\begin{figure}[h] 

  \centering
    \includegraphics[scale=0.4]{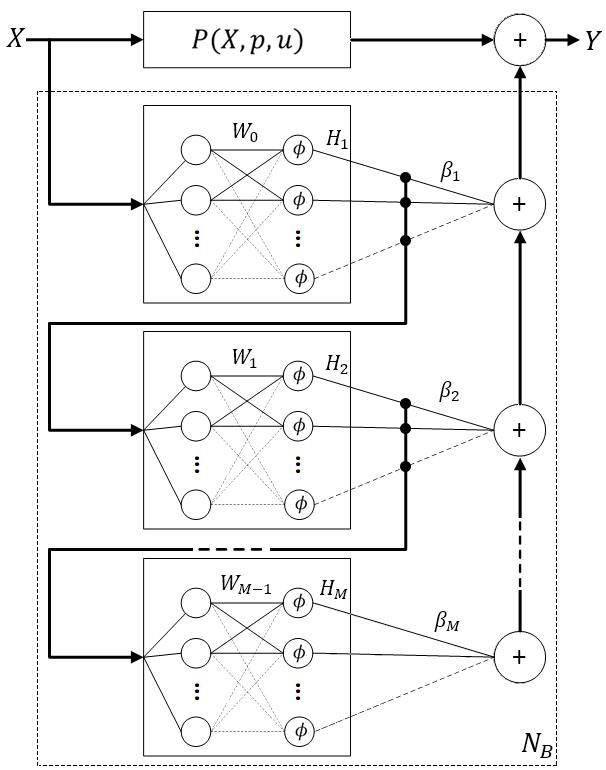}
    \caption{Visual representation of the SCM model
    }
    \label{fig:SCM_model}
\end{figure}
\subsection{Model Capacity}

Theoretically, randomized neural networks with a class of activation functions can approximate any continuous function with probability one \cite{471375}. Unfortunately, such a fundamental result cannot be used to direct the training process for successful data modelling, in particular, as the scope of random weights and biases is fixed\cite{8013920}. Indeed, the capacity of a randomized learner can be characterized by model complexity, which measures both the richness of the local extreme points and the average sum of the absolute values of the partial derivatives of the learner model over the domain. This section firstly introduces the model complexity concept, and uses it to guide the design of SCM.\\

\begin{definition}For an SCM, $F(X)=P(X,p,u)+S(X)$ and $S(X)=\sum_{k=1}^{M}\beta_k H_k(X)$, where $X\in[a,b]$, the model complexity (MC) is defined by
\begin{equation}\label{mc_1}
    MC(F)=Z(S')\int_{a}^{b}|S'(X)|dX
\end{equation}
where $S'$ represents the derivative of $S(X)$ and $Z(S')$ stands for the number of local extremums of S in $[a,b]$. 
For the multivariable case, we can generalize (\ref{mc_1}) as follows:
\begin{equation}
    MC(F)=Z(\nabla S)\int_{a_1}^{b_1}\cdots\int_{a_n}^{b_n}\sum_{i=1}^{n}\left|\frac{\partial S}{\partial X_i}\right|dX_1...dX_n
\end{equation}
where $\nabla S=\left( \frac{\partial S}{\partial X_1}, \frac{\partial S}{\partial X_2}, \cdots, \frac{\partial S}{\partial X_n} \right)$, $\frac{\partial S}{\partial X_i}$ represents the partial derivative of $S$ with respect to $X_i$, and  $Z(\nabla S)$ stands for the number of local extremums of S over $[a_1,b_1]\times...\times[a_n,b_n]$.\\
\end{definition}
\emph{\textbf{Remark 2:}} The MC concept is valid for any differentiable functions rather than SCM models, where the first term is purposely ignored. \\
\begin{definition}
Consider a real-valued differentiable function $f:\mathbb{R}^n\to\mathbb{R}$ defined over a compact set. 
An SCM learner model $F(X;\theta)$ (here $\theta$ denotes the set of model's parameters) is said to hold one-order universal approximation property if,
for any arbitrary $\varepsilon>0$, there exists a $\theta^*$ such that  $||F(X;\theta^*)-f(X)||<\varepsilon$ and $\left\|\frac{\partial F(X;\theta^*)}{\partial X_i}-\frac{\partial f(x)}{\partial X_i}\right\|<\varepsilon, i=1,2,...,n$.\\
\end{definition}
\begin{theorem}
Given a real-valued differentiable function $f(X)$ defined over $[a_1,b_1]\times [a_2,b_2]\times ... \times[a_n,b_n]$. An SCM model $F(X)$ does not hold the one-order universal approximation property to $f(X)$ if $MC(F) < MC(f)$.
\end{theorem}
\begin{proof}
Let $F(X)=f(X)+e(X)$ be an SCM model that holds the one-order universal approximation property to $f(X)$. According to Definition 3.1, for any arbitrary $\varepsilon > 0$, we have
\begin{equation}
    ||e(X)||<\varepsilon 
   \text{ and }\left\|\frac{\partial e(X)}{\partial X_i}\right\|<\varepsilon, i=1,2,...,n.
\end{equation}
Therefore
\begin{equation}
\begin{split}
    MC(f) & = MC(F-e)\\
    & =Z(\nabla (S-e))\int_{a_1}^{b_1}\cdots\int_{a_n}^{b_n}\sum_{i=1}^n\left|\frac{\partial(F-e)}{\partial X_i}\right|dX_1 ... dX_n\\
    & = Z(\nabla S- \nabla e) \int_{a_1}^{b_1}\cdots\int_{a_n}^{b_n}\sum_{i=1}^n\left|\frac{\partial S}{\partial X_i}-\frac{\partial e}{\partial X_i}\right|dX_1 ... dX_n
\end{split}
\end{equation}
Notice that $\nabla e =\left( \frac{\partial e}{\partial X_1}, \frac{\partial e}{\partial X_2}, \cdots, \frac{\partial e}{\partial X_n}\right)$,
$\left\| \frac{\partial e}{\partial X_i} \right\|<\varepsilon$,
$i=1,2,...,n$, and $\varepsilon$ can be arbitrarily small. It is easy to prove that
\begin{equation}
    Z\left(\nabla S- \nabla e\right)=Z\left(\nabla S \right).
\end{equation}
Also, for any X, the following inequality holds, that is,
\begin{equation}
    \left| \frac{\partial S}{\partial X_i}-\frac{\partial e}{\partial X_i} \right| \le \left|\frac{\partial S}{\partial X_i} \right| + \left|\frac{\partial e}{\partial X_i} \right|, i=1,2,...,n.
\end{equation}
Thus, we get
\begin{equation}
\begin{split}
MC(f) & \le Z(\nabla S)\int_{a_1}^{b_1}\cdots\int_{a_n}^{b_n}\sum_{i=1}^n \left|\frac{\partial S}{\partial X_i} \right| dX_1...dX_n
+Z(\nabla S)\int_{a_1}^{b_1}\cdots\int_{a_n}^{b_n}\sum_{i=1}^n\left|\frac{\partial e}{\partial X_i}\right|dX_1...dX_n\\
& \le MC(F)+nZ(\nabla S)\left(\prod^n_{i=1}(b_i-a_i)\right)^{\frac{1}{2}}\sum_{i=1}^n\left(\int_{a_1}^{b_1}\cdots\int_{a_n}^{b_n} \left(\frac{\partial e}{\partial X_i} \right)^2 dX_1...dX_n\right)^{\frac{1}{2}}\\
& \le MC(F)+n^2\varepsilon  Z(\nabla S) \left( \prod^n_{i=1}(b_i-a_i)\right)^{\frac{1}{2}}
\end{split}
\end{equation}
Let $\varepsilon\to0$, we have
\begin{equation}
MC(f) \le MC(F)
\end{equation}
This contradicts and completes the proof.
\end{proof}

\begin{theorem}
Given a real-valued differentiable function $f:[a_1,b_1]\times [a_2,b_2]\times ... \times[a_n,b_n] \to \mathbb{R}$, and a class of SCM models with a differentiable activation function $\phi$. A necessity condition on the one-order universal approximation property of SCM with respect to f is that there exists at least one $F^*$ such that $MC(F^*)\ge MC(f)$.
\end{theorem}
\begin{proof}
The consequence comes immediately from Theorem 3.1.
\end{proof}

\emph{\textbf{Remark 3:}} It is yet unclear if the condition in Theorem 3.2 is sufficient. Indeed, the theoretical results reported above are meaningful although the one-order universal approximation property is much stronger than the original (or termed as zero-order) universal approximation property. Further research in this direction is to extend the presented results to zero-order universal approximation property, that is, the relationship between the model's complexity and its representation power to nonlinear functions. Intuitively, if a learner model enables us to approximately express a given function or a data set, the degree of complexities from both the model and the data should fit with each other. For instance, one cannot expect a well-fit oscillatory curve by a linear regression model, regardless of whatever parameters we adjust. From this understanding, we make the following conjectures, which require more effort to rigorously prove. \\

\begin{conjecture}
Given a real-valued differentiable function $f:[a_1,b_1]\times [a_2,b_2]\times ... \times[a_n,b_n] \to \mathbb{R}$, and a class of SCM models with a differentiable activation function $\phi$. A necessity condition on the universal approximation property of SCM with respect to f is that there exists at least one $F^*$ such that $MC(F^*)\ge MC(f)$.\\
\end{conjecture}
\begin{conjecture}
Given a real-valued continuous function $f:[a_1,b_1]\times [a_2,b_2]\times ... \times[a_n,b_n] \to \mathbb{R}$, and a class of SCM models with a bounded activation function $\phi$. A necessity condition on the one-order universal approximation property of SCM with respect to f is that there exists at least one $F^*$ such that $MC(F^*)\ge MC(f)$.
\end{conjecture}

\subsection{Learning Algorithm}\label{alg_desc}
\begin{figure}[h]

  \centering
    \includegraphics[scale=0.7]{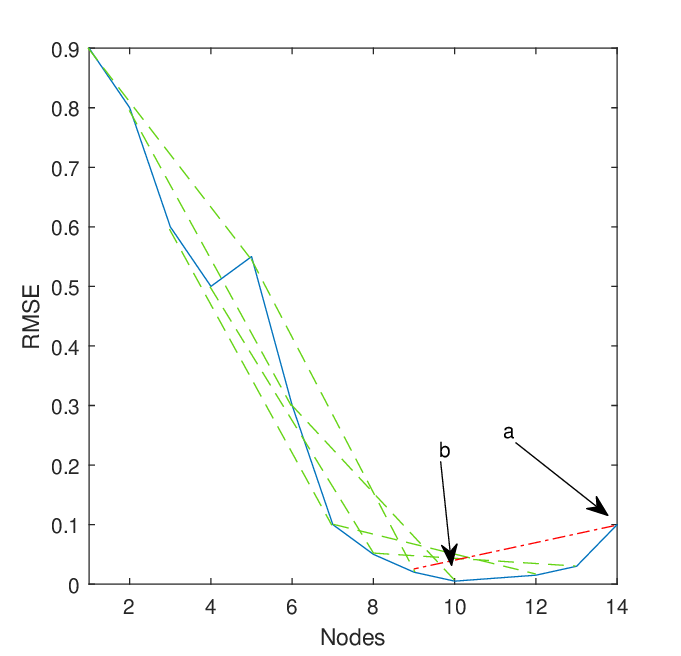}
    \caption{Early stopping demonstration }
    \label{fig:early stopping}
\end{figure}

Given a training dataset with $d$ features, $m$ outputs, $N$ training examples, that is, $X_t =[x_{t1}, x_{t2},..., x_{tN}], x_{ti}=[x_{ti,1},...,x_{ti,d}]^T \in \mathbb{R}^d$ and outputs $Y_t = [y_{t1}, y_{t2},..., y_{tN}], y_{ti}=[y_{ti,1},...,y_{ti,m}]^T \in \mathbb{R}^m$.
Then the weights of the linear regression model are given by $p^*=[p_{1}^*,...,p_{m}^*], p_i^*=[p_{i,1}^*,...,p_{i,d}^*]^T$. In our algorithm, these are found by using LASSO regression \cite{tibshirani1996regression}, that is,
\begin{equation} \label{eq:cognitive}
p_i^*=\argmin_{p_i}\left(\sum^N_{j=1}(y_{tj,i}-\sum^d_{k=1}x_{tj,k}p_{i,k})^2+\alpha \sum^d_{k=1}|p_{i,k}|\right), i=1,2,...,m
\end{equation}
The intercept term of each $p_i^*$ is given by $u^*=[u_{1}^*,...,u_{m}^*]$ by finding the mean of each $y_{t:,i}$, that is, $u_i=\frac{\sum^N_{j=1}y_{tj,i}}{N}$, $i=1,2,...,m$. 
Note that the determination of the weights is not limited to the linear regression model, any suitable linear models could be used, as demonstrated in Section 3.
The residual training error vector before adding the \emph{L-th} hidden node on the \emph{n-th} hidden layer is added, donated by  $\mathcal{E}_{L_n-1}^{(n)} = \mathcal{E}_{L_n-1}^{(n)}(X_t) = [\mathcal{E}_{L_n-1,1}^{(n)}(X_t),...,\mathcal{E}_{L_n-1,m}^{(n)}(X_t)]^T$, where $\mathcal{E}_{L_n-1,q}^{(n)}(X_t) = [\mathcal{E}_{L_n-1,q}^{(n)}(x_{t1}),...,\mathcal{E}_{L_n-1,q}^{(n)}(x_{tN})]^T \in  \mathbb{R}^N, q=1,2,...,m$.
Then, after adding the \emph{L-th} hidden node in the \emph{n-th} hidden layer, we can calculate the output of the \emph{n-th} hidden layer:
\begin{equation}\label{eq:hidden_out}
h_{L_n}^{(n)}:=h_{L_n}^{(n)}(X_t)= [\phi_{n,L_n}(x_1^{(n-1)}),...,\phi_{n,L_n}(x_N^{(n-1)})]^T 
\end{equation}
where $\phi_{n,L_n}(x_i^{(n-1)})$ is used to simplify $\phi_{n,L_n}(x_i^{(n-1)}, w_j^{(n-1)} , b_j^{(n-1)} )$ , and $x_i^{(0)}=x_i=[x_{i,1},...,x_{i,d}]^T, x_i^{(n-1)}=\Phi(x^{(n-2)}, W^{(n-1)}, B^{(n-1)})$ for $n \ge 2$. \\
Let $H_{L_n}^{(n)}=[h_1^{(n)},h_2^{(n)},...,h_{L_n}^{(n)}]$ represent the hidden layer output martix. A temporary variable $\xi_{L_n,q}^{(n)} (q=1,2,...,m)$ is introduced: 
\begin{equation}\label{eq:theta}
\xi_{L_n,q}^{(n)}=\frac{\langle \mathcal{E}_{L_n-1,q}^{(n)} , h_{L_n}^{(n)} \rangle^2}{\langle h_{L_n}^{(n)}, h_{L_n}^{(n)} \rangle}-(1-r)\langle \mathcal{E}_{L_n-1,q}^{(n)}, \mathcal{E}_{L_n-1,q}^{(n)}\rangle.
\end{equation}

\begin{algorithm*}
\caption{SCM with Early Stopping}
\label{alg1}
\SetKwInOut{Parameters}{Parameters}
\SetKwInOut{Returns}{Returns}
\SetKwInOut{Input}{Input}
\SetKwInOut{Output}{Output}

\SetKwFunction{FSCM}{SCM}
\SetKwFunction{FSC}{SCsearch}
\SetAlgoLined

\Input{ $\>$Training inputs $X_t = [x_{t1}, x_{t2},..., x_{tN}], x_{ti} \in \mathbb{R}^d$, Training outputs $Y_t = [y_{t1}, y_{t2},..., y_{tN}], y_{ti} \in \mathbb{R}^m$; 
       \tabto{2mm} Testing inputs $X_v = [x_{v1}, x_{v2},..., x_{vK}], x_{vi} \in \mathbb{R}^d$, Testing outputs $Y_v = [y_{v1}, y_{v2},..., y_{vK}], y_{vi} \in \mathbb{R}^m$; }
\Parameters{ $\>$Max hidden Layers $M$; Max Hidden Nodes per Layer $L_{max}^{(n)}=\{L|L=\infty\}, 1\leq n\leq M$; 
       \tabto{2mm} Error tolerance $\epsilon$; Candidates per Layer $T_{max}^{(n)}, 1\leq n\leq M$; 
       \tabto{2mm} Two sets of scalars $\Upsilon=\{\lambda_1,...,\lambda_{end} \}, \mathcal{R}=\{r_1,...,r_{end}\}$; 
       \tabto{2mm} Early stopping tolerance $\tau$; early stopping step $L_{step}$; LASSO regularization factor $\alpha$;
       }
\Output{ Weight scale $\Upsilon^*$, $\>$Output weights $\beta^*$, Hidden weights $w^*$, Hidden biases $b^*$, linear weights $p^*$ and intercept $u^*$}
\SetKwProg{Fn}{Function}{:}{}
\Fn{\FSCM{$X_t, T_t, X_v, T_v, M, \epsilon, L_{max}^{(n)}, T_{max}^{(n)}, \Upsilon
, \mathcal{R}, \tau, L_{step}, $}}{

 $\mathcal{E}_0^{(1)}:=[t_{t1}, t_{t2},..., t_{tN}]^T, \mathcal{H}:=[ ], \Omega:=[ ], W:=[ ]$, $a = d$\;
 \For{$i=1,2,...,m$}
 {
    Find $p_i^*$ by (\ref{eq:cognitive}) and $u_i^*$ by taking the mean of $y_{ti}$\;
 }
 
 \While{$n\le M$ and $||\mathcal{E}_{0 F}^{(1)}||>\epsilon$}
 {
    \While{$L_n\le L_{max}^{(n)}$ and $||\mathcal{E}_{0 F}^{(1)}||>\epsilon$}
    {
        \For{$\lambda\in\Upsilon$}
        {
            \For{$r\in\mathcal{R}$}
            {
                \For{$k=1,2,...,T_{max}^{(n)}$}
                {
                    Randomly assign values to $w_{rand}$ (from: binary $\{-1,1\}^a$ ) and  $b_{rand}$ (from: real $[-1,1]$)\;
                    $w_{L_n}^{n-1}$ = $\lambda\cdot w_{rand}$, \ 
                    $b_{L_n}^{n-1}$ = $\lambda\cdot b_{rand}$,
                    and calculate $h_{L_n}^{(n)}, \xi_{L_n,q}^{(n)}$ by (\ref{eq:hidden_out}) and (\ref{eq:theta})\;
                    
                    \uIf{$min\{\xi_{L,1}^{(n)},...,\xi_{L,m}^{(n)}\}>0$}
                    {
                        Save $w_{L_n}^{(n-1)}$ and $b_{L_n}^{(n-1)}$ in $W$, 
                        $\xi_{L_n}^{(n)}=\sum_{q=1}^{m} \xi_{L_{n,q}}^{(n)}
                         \ in\ \Omega$\;
                    }
                }
                \If{$W$ is not empty}
                {
                    Find $w_{L_n}^{(n-1)\star}$, $b_{L_n}^{(n-1)*}$ 
                    maximizing $\xi_{L,m}^{(n)}$, set
                    $H_{L_n}^{(n)}=[h_1^{(n)*},...,h_{L_n}^{(n)*}]$\;
                    Set $w_{L_n}^{(n-1)*}=1/\lambda \cdot w_{L_n}^{(n-1)\star}$, 
                    and store $\lambda$ in $\Upsilon_{L_n}^{(n-1)*}$\;
                    \bf{Break} (go to Line 23)\;
                }
            }
        }
        Set $\mathcal{H}:=[\mathcal{H}, H_{L_n}^{(n)}],\ \beta^*=\mathcal{H}^\dagger (Y_t^T-(X_t^Tp^*+I_1u^*)) ,\ 
        \mathcal{E}_{L_n}^{(n)}=\mathcal{H}\beta^*+X_t^Tp^*+I_1u^*-Y_t^T,\ \mathcal{E}^{(n)} = \mathcal{E}_{L_n}^{(n)},\ a=L_n $\;
        
        $E_{L_n}^{(n)}$ = $Y_v^T-(X_v^Tp^*+I_1u^*+\sum_{k=1}^{n}\beta^*_kH_k(X_v))$ for given $\Upsilon^*$, $\beta^*$, $\omega^*$, $b^*$, $p^*$, $u^*$\; 
        \If{$L_n>L_{step}$}
        {
            \If{$\frac{E_{L_n-L_{step}}^{(n)}-E_{L_n}^{(n)}}{E_{L_n}^{(n)}}\le\tau$}
            {
                    \Repeat{$\frac{E_{L_n-1}^{(n)}-E_{L_n}^{(n)}}{E_{L_n}^{(n)}}>\tau$}
                    {
                        \bf{Undo} line 23\;
                    }
                    \bf{break} (go to line 35)\;
            }
        
        }
        Renew $\mathcal{E}_0^{(1)}:=\mathcal{E}_{L_n}^{(n)}$, $L_n:=L_n+1$\;
    }
    Set $\mathcal{E}_0^{(n+1)}:=\mathcal{E}^{(n)}$, $\Omega:=[ ]$, $W:=[ ]$\; 
    Renew $n := n + 1$\;
 }
 \bf{Return} $\Upsilon^*$, $\beta^*$, $w^*$, $b^*$, $p^*$, $u^*$\;
}
\end{algorithm*}

In the constructive approach, the question of when to stop the addition of nodes and when to add the next hidden layer arises. In DeepSCN \cite{DBLP:journals/corr/WangL17c}, the approach is taken where the number of hidden nodes is set to a maximum and a new hidden layer is added when this maximum is reached or if no suitable candidate nodes can be found.
In this paper, we add a third condition for stopping adding nodes and adding a new layer. 
Given a validation dataset with K examples, with inputs $X_v =[x_{v1}, x_{v2},..., x_{vK}], x_{ti}=[x_{vi,1},...,x_{vi,d}]^T \in \mathbb{R}^d$ and outputs $Y_v = [y_{v1}, y_{v2},..., y_{vK}], y_{vi}=[y_{vi,1},...,y_{vi,m}]^T \in \mathbb{R}^m$.
The residual error vector after adding the \emph{L-th} hidden node on the \emph{n-th} hidden layer is donated by  $E_{L_n}^{(n)} = E_{L_n}^{(n)}(X_t)$.
A step size $L_{step}$ and a tolerance $\tau$ are used in the early stopping criterion. If the condition $\frac{E_{L_n-L_{step}}^{(n)}-E_{L_n}^{(n)}}{E_{L_n}^{(n)}}\le\tau$ $(L_n>L_{step})$ is met, then the nodes are iteratively removed until $\frac{E_{L_n-1}^{(n)}-E_{L_n}^{(n)}}{E_{L_n}^{(n)}}>\tau$  is satisfied and a new layer is added. The justification behind this is that if a testing local minimum is reached or the training does not have much of an effect on the testing, then the model may find a better representation using another layer. This, in turn, ideally improves the testing results and prevents the over-fitting phenomenon. Further, a step size is used as it is expected that the expressive power of a layer that has very few nodes is low and noise is present. This is illustrated in Figure \ref{fig:early stopping}, where $L_{step}=5$ and $\tau=0$, point (a) shows where condition $\frac{E_{L_n-L_{step}}^{(n)}-E_{L_n}^{(n)}}{E_{L_n}^{(n)}}\le\tau$ is true, and
point (b) shows from this node (node 14), nodes are iteratively removed until $\frac{E_{L_n-1}^{(n)}-E_{L_n}^{(n)}}{E_{L_n}^{(n)}}>\tau$. Hence, in this example, nodes 11,12,13 and 14 would be removed.\\

To calculate the output weights, let $\mathcal{H} = [H_{L_1}^{(1)}, H_{L_2}^{(2)},...,H_{L_M}^{(M)}]\in \mathbb{R}^{N\times \sum_{k=1}^ML_k}$ represent the output matrix of all hidden layer, where $L_k, k=1,2,...,M$, represents the number of nodes at the \emph{k-th} layer, respectively. Then the optimal solution $\beta^*$ is computed using the least squares method as follows:
\begin{equation}\label{eq:beta}
\beta^*= \argmin_{\beta}||\mathcal{H}\beta-(Y_t^T-(X_t^Tp^*+I_1u^*))||^2_F=\mathcal{H}^\dagger (Y_t^T-(X_t^Tp^*+I_1u^*))
\end{equation}
where $\mathcal{H}^\dagger$ is the Moore-Penrose generalised inverse \cite{lancaster1985theory} of the matrix $\mathcal{H}$, $I_1=[1,1,...,1]^T\in\mathbb{R}^N$ and $||\cdot||_F$ denotes the Frobenius norm.\\

The hidden weights in SCM are defined to be binary, and are scaled by the adaptive scope parameter $\lambda$, resulting in a floating point value. To reduce memory as $\lambda$ is the same value for a given node, the $\lambda$ value is stored for each node such that given the \emph{n-th} layer: 
\begin{equation}
\Upsilon_{n}^*=[\lambda^n_1,\lambda^n_2, ..., \lambda_{L_n}^n], 
\end{equation}
and in turn the weights can be stored in binary and are scaled only when fed forward such that:
\begin{equation}
W_{n}^* = \Upsilon_{n}^*\cdot w_{n}^* .
\end{equation}

\section{Performance Evaluation} \label{exper}
This section reports our results over both benchmark and real-world industrial datasets. Performance is evaluated on learning, generalization and efficiency. The models discussed are implemented in Matlab, and run on a PC with an Intel Core i7-3820 @ 3.6GHz and 32GB of ram. All models are implemented into a unified framework for comparison. Comparisons are made among SCN, DeepSCN, SCM, IRVFL, a deep version of IRVFL termed DIRVL-I and a deep version of IRVFL with a linear model termed DIRVFL-II. 

\subsection{Experimental Setup}
\begin{table}[htbp]
  \centering
  \caption{Summary of each algorithms features}
    \begin{tabular}{p{50pt}|p{35pt}|p{35pt}|p{45pt}|p{45pt}}
    
    \hline
    Algorithm & Deep & Early-Stopping & Linear \par Model & Supervisory \par Mechanism \bigstrut[b]\\
    \hline
    \hline
    SCN & N & N & N & Y \bigstrut[t]\\
    DeepSCN & Y & N & N & Y \\
    SCM & Y & Y & Y & Y \\
    IRVFL & N & N & N & N \\
    DIRVFL-I & Y & N & N & N \\
    DIRVFL-II & Y & Y & Y & N \\
    \hline
    \end{tabular}%
 
  \label{tab:alg_sum}%
\end{table}%

The framework builds the network one hidden node at a time for every algorithm. Before a node is added, the hidden weights are determined using a supervisory mechanism for SCN, DeepSCN and SCM, or purely randomly for IRVFL, DIRVFL-I and DIRVFL-II. The hidden weights can be constrained to binary $\{-1,1\}$ or real $[-1,1]$ values for all algorithms except SCM and DIRVFL-II, which are limited to binary. 
For SCN, DeepSCN, IRVFL and DIRVFL-I, the maximum number of hidden nodes is set to ensure results are reported within a reasonable time. However, SCM and DIRVFL-II use an early stopping method as outlined in Section. \ref{alg_desc}; hence, the maximum number of nodes is ideally set to infinity. For practical purposes, this is just set to a very large value. SCM and DIRVFL-II both use a mechanism model using lasso regression unless specified otherwise. Table.\ref{tab:alg_sum} shows the differences and similarities for convenience.

For Binary IRVFL, DIRVFL-I and DIRVFL-II the weights ($w$) are selected randomly from the values $\{-1, 1\}$, and the biases ($b$) are selected randomly from [-1,1] as 64 bit floating point numbers. For Binary SCN, DeepSCN and SCM the candidate weights ($w$) are selected randomly from the values $\{-1, 1\}$ and are multiplied by the scaling factor ($\lambda$); the biases ($b$) are selected randomly from $[-1,1]$ as 64 bit floating point numbers and are also multiplied by the scaling factor ($\lambda$) selected from $\{0.5,1,5,10,30,50,100\}$. The maximum number of candidate weights and biases is set based on the dataset described in later sections.

The activation functions used are all bounded and chosen based on performance, where the activations that appear to perform best on the dataset from experimentation are chosen. The activation functions used are:
 \begin{itemize}
\item Sigmoidal activation function given by $\phi(x)=\frac{1}{1+e^{-x}}$
\item Bounded Rectified Linear Unit\cite{bounded} (BRELU) given by $\phi(x)=min(max(0,x),A)=\left\{
\begin{array}{ll}
0 &  x\le 0 \\[-2pt]
x &  0<x\le A \\[-2pt]
A &  x>A
\end{array}
\right.$
 where we have used $A=1$.
\item Hyperbolic tangent function (tanh) given by $\phi(x)=tanh(x)$
\item Binary sign function given by 
$\phi(x)=\left\{
\begin{array}{ll}
-1 &  x\le 0 \\[-1pt]
1 &  x > 0 
\end{array}
\right.$
\item Hard limiter function given by 
$\phi(x)=\left\{
\begin{array}{ll}
1 &  x\ge 0 \\[-1pt]
0 &  x < 0 
\end{array}
\right.$
\end{itemize}
\subsection{Datasets}

Benchmark datasets include eight regression and two classification datasets either generated or downloaded from the UCI Machine Learning Repository\cite{Dua:2019} or Zalando\cite{xiao2017_online}. Three industrial datasets are used to
demonstrate SCM use in real-world applications. For all datasets, input attributes and output targets were all normalized between [0,1], and  randomly split into training and testing, at 90\% and 10\%, respectively.
\begin{itemize}
\item The regression task of R-DB1 is to predict the age of abalone from physical attributes, namely sex, length, diameter, height, weight (whole, shucked, viscera and shell) and rings. 
\item R-DB2's task is to predict the compressive strength of concrete from the components in the mixture, namely cement, blast furnace slag, fly ash, water, superplasticizer, course aggregate and fine aggregate and the age of the concrete. The dataset contains 1030 instances.
\item R-DB3 is a collection of properties of power plant running at full load, consisting of temperature, pressure, humidity and exhaust vacuum. The target of this dataset is the electrical energy output. The dataset contains 9500 instances.
\item R-DB4 consists of US Census Data of housing in Boston, Massachusetts, each point is a collection of houses in a town. Attributes are crime rate, residential land, non-retail buisness land, nitric oxide concentration, average rooms per dwelling, units built before 1940, distance to employment centres, access to highways, property tax rate, pupil-teacher ratio, blacks by town and percent lower status of population. The target is to predict the median price of the houses for a particular town. The dataset contains 450 instances.
\item The task of R-DB5 is to predict the popularity of a topic on the micro-blogging platform Twitter. The dataset contains 77 features, such that there are 11 primary features, each made up of 6 features that describe the feature through time. These primary features are the number of created discussions, author increase, attention level, burstiness level, number of atomic containers,  attention level measured with contributions, contribution sparseness, author interaction, number of authors, average discussion length, and the number of discussions. 
The dataset contains 38393 instances.
\item R-DB6 task is to predict the year of a song from audio features. 90 features describe the music based on timbre. The dataset contains 515,345 instances.
\item R-DB7 is a dataset based on the real value function \cite{tyukin2009feasibility} defined as
\begin{equation}
f(x)=0.2e^{-(10x-4)^2}+0.5e^{-(90x-40)^2}+0.3e^{-(80x-20)^2}
\end{equation}
The dataset contains 1000 instances generated from the uniform distribution [0,1].
\item R-DB8 is a dataset generated from the Rastrigin function \cite{10.1007/BFb0029787} defined as 
\begin{equation}
f(x)=An+\sum_{i=1}^{n}[x_i^2-A\cos{2\pi x_i}]
\end{equation}
where $A=10$, $x_i\in[-5.12,5.12]$. A $n$ of 2 is used, with 40000 training instances and 4489 testing instances. 
\item C-DB1 is the widely used MINST database. The target is handwritten digits from 28*28 pixel greyscale images. The dataset contains 70000 instances.
\item C-DB2 task is a database of 28*28 pixel greyscale images of fashion items. The target is the type of fashion item. The dataset contains 70000 instances.
\item The dataset I-DB1 \cite{LI2022677} was obtained from sensors from steel production plants using 12 different target thicknesses. 3163 samples are used, 2863 in the training and 300 in testing. Feature selection involves using grey relational analysis \cite{liu2017grey} to ascertain the most suitable inputs. The inputs include 8 roll gap measurements, entrance thickness, entrance temperature, exit temperature, strip width, 8 rolling force measurements, and 8 roll linear speed measurements. Further, 8 inputs are calculated based on the mechanism model for roll wear. The target is the measurement of the strip thickness. 
\item The dataset I-DB2 contains the parameters and the measured feedback signals of a servo control system. The measurements consist of the speed signal, the current from the motor and the current from the speed controller at 3 incremental points in time. The parameters are the servo target speed and the servo stiffness. The dataset's target is the error between the current from the motor and the current of the servo control system. The dataset is obtained by performing real-world testing with a load, where 8407 different target speed and stiffness configurations are tested. This dataset involves a very large amount of data, with 5,145,084 samples for training and 1,715,028 for testing.
\item The dataset I-DB3, similar to I-DB1, concerns hot-rolling of steel; however, it aims to predict the rolling force for plates with varying thicknesses. This dataset includes 14 input parameters, including entrance thickness, exit thickness, entrance width, rolling speed, temperature, various measurements of the content of particular elements, and plan view pattern control parameters. Notably, this dataset includes a mechanism model based on key parameters in the production process; this is used to demonstrate the benefit of a mechanism model used with SCM.
\end{itemize}

\subsection{Results for benchmark datasets}
The benchmark datasets are tested only using binary weights to demonstrate the differences between the algorithms with data reduction. \\

\begin{table*}[htbp]

  \centering
  \caption{Model parameter settings for benchmark datasets}
  \resizebox{\textwidth-60pt}{!}{
    \begin{tabular}{c|c|c|c|c|c|c|c}
    \hline
    Data Set & Algorithm & Layers & $L_{max}$ & Activations & $T_{max}$ & $L_{step}$ & $\tau$ \bigstrut[b]\\
    \hline
    \hline
    \multirow{6}[1]{*}{R-DB1} & SCN & 1 & 50 & S & 700 & - & - \bigstrut[t]\\
      & DeepSCN & 2 & 25,25 & S,S & 500,700 & - & - \\
      & SCM & 2 & $\infty$,$\infty$ & S,S & 500,700 & 10 & 0.001 \\
      & IRVFL  & 1 & 50 & S & - & - & - \\
      & DIRVFL-I & 2 & 25,25 & S,S & - & - & - \\
      & DIRVFL-II & 2 & $\infty$,$\infty$ & S,S & - & 10 & 0.001\bigstrut[b]\\
    \hline
    \multirow{6}[1]{*}{R-DB2} & SCN & 1 & 50 & T & 900 & - & - \\
      & DeepSCN & 5 & 20,20,20,20,20 & T,T,T,T,T & 500,600,700,800,900 & - & - \\
      & SCM & 5 & $\infty$,$\infty$,$\infty$,$\infty$,$\infty$ & T,T,T,T,T & 500,600,700,800,900 & 10 & 0.001 \\
      & IRVFL  & 1 & 50 & T & - & - & - \\
      & DIRVFL-I & 5 & 20,20,20,20,20 & T,T,T,T,T & - & - & - \\
      & DIRVFL-II & 5 & $\infty$,$\infty$,$\infty$,$\infty$,$\infty$ & T,T,T,T,T & - & 10 & 0.001 \bigstrut[b]\\
    \hline
    \multirow{6}[2]{*}{R-DB3} & SCN & 1 & 50 & T & 900 & - & - \bigstrut[t]\\
      & DeepSCN & 5 & 20,20,20,20,20 & T,T,T,T,T & 500,600,700,800,900 & - & - \\
      & SCM & 5 & $\infty$,$\infty$,$\infty$,$\infty$,$\infty$ & T,T,T,T,T & 500,600,700,800,900 & 10 & 0.001 \\
      & IRVFL  & 1 & 50 & T & - & - & - \\
      & DIRVFL-I & 5 & 20,20,20,20,20 & T,T,T,T,T & - & - & - \\
      & DIRVFL-II & 5 & $\infty$,$\infty$,$\infty$,$\infty$,$\infty$ & T,T,T,T,T & - & 10 & 0.001 \bigstrut[b]\\
    \hline
    \multirow{6}[2]{*}{R-DB4} & SCN & 1 & 25 & S & 700 & - & - \bigstrut[t]\\
      & DeepSCN & 3 & 20,20,20 & S,S,S & 500,600,700 & - & - \\
      & SCM & 3 & $\infty$,$\infty$,$\infty$ & S,S,S & 500,600,700 & 10 & 0.005 \\
      & IRVFL  & 1 & 25 & S & - & - & - \\
      & DIRVFL-I & 3 & 20,20,20 & S,S,S & - & - & - \\
      & DIRVFL-II & 3 & $\infty$,$\infty$,$\infty$ & S,S,S & - & 10 & 0.005 \bigstrut[b]\\
    \hline
    \multirow{6}[1]{*}{R-DB5} & SCN & 1 & 30 & S & 700 & - & - \bigstrut[t]\\
      & DeepSCN & 3 & 20,20,20 & S,S,S & 500,600,700 & - & - \\
      & SCM & 3 & $\infty$,$\infty$,$\infty$ & S,S,S & 500,600,700 & 10 & 0.001 \\
      & IRVFL  & 1 & 30 & S & - & - & - \\
      & DIRVFL-I & 3 & 20,20,20 & S,S,S & - & - & - \\
      & DIRVFL-II & 3 & $\infty$,$\infty$,$\infty$ & S,S,S & - & 10 & 0.001 \bigstrut[b]\\
    \hline
    \multirow{6}[1]{*}{R-DB6} & SCN & 1 & 50 & B & 1200 & - & - \\
      & DeepSCN & 3 & 30,30,30 & B,B,B & 1000,1100,1200 & - & - \\
      & SCM & 3 & $\infty$,$\infty$,$\infty$ & B,B,B & 1000,1100,1200 & 10 & 0.001 \\
      & IRVFL  & 1 & 50 & B & - & - & - \\
      & DIRVFL-I & 3 & 30,30,30 & B,B,B & - & - & - \\
      & DIRVFL-II & 3 & $\infty$,$\infty$,$\infty$ & B,B,B & - & 10 & 0.001 \bigstrut[b]\\
    \hline
    \multirow{6}[2]{*}{R-DB7} & SCN & 1 & 50 & T & 1100 & - & - \bigstrut[t]\\
      & DeepSCN & 2 & 50,50 & T,T & 1000,1100 & - & - \\
      & SCM & 2 & $\infty$,$\infty$ & T,T & 1000,1100 & 10 & 0.001 \\
      & IRVFL  & 1 & 50 & T & - & - & - \\
      & DIRVFL-I & 2 & 50,50 & T,T & - & - & - \\
      & DIRVFL-II & 2 & $\infty$,$\infty$ & T,T & - & 10 & 0.001 \bigstrut[b]\\
    \hline
    \multirow{6}[1]{*}{R-DB8} & SCN & 1 & 100 & T & 1000 &   & - \bigstrut[t]\\
      & DeepSCN & 10 & 50,...,50 & T,...,T & 100,200,...,900,1000 & - & - \\
      & SCM & 10 & $\infty$,..,$\infty$ & T,...,T & 100,200,...,900,1000 & 10 & 0.003 \\
      & IRVFL  & 1 & 100 & T & - & - & - \\
      & DIRVFL-I & 10 & 50,...,50 & T,...,T & - & - & - \\
      & DIRVFL-II & 10 & $\infty$,..,$\infty$ & T,...,T & - & 10 & 0.003 \bigstrut[b]\\
    \hline
    \multirow{6}[1]{*}{C-DB1} & SCN & 1 & 200 & T & 600 & - & - \\
      & DeepSCN & 3 & 100,100,100 & T & 400-500-600 & - & - \\
      & SCM & 3 & $\infty$,$\infty$,$\infty$ & T,T & 400-500-600 & 10 & 0.004 \\
      & IRVFL  & 1 & 200 & T & - & - & - \\
      & DIRVFL-I & 3 & 100,100,100 & T & - & - & - \\
      & DIRVFL-II & 3 & $\infty$,$\infty$,$\infty$ & T,T & - & 10 & 0.004 \bigstrut[b]\\
    \hline
    \multirow{6}[2]{*}{C-DB2} & SCN & 1 & 150 & R & 3000 & - & - \bigstrut[t]\\
      & DeepSCN & 3 & 80,80,80 & R,R,R & 2000-2500-3000 & - & - \\
      & SCM & 3 & $\infty$,$\infty$,$\infty$ & R,R,R & 2000-2500-3000 & 10 & 0.003 \\
      & IRVFL  & 1 & 150 & R & - & - & - \\
      & DIRVFL-I & 3 & 80,80,80 & R,R,R & - & - & - \\
      & DIRVFL-II & 3 & $\infty$,$\infty$,$\infty$ & R,R,R & - & 10 & 0.003 \bigstrut[b]\\
    \hline
    \multicolumn{4}{p{250pt}}{$^{\mathrm{a}}$ S = Sigmoid, T = Tanh, B = Binary Sign, R = Bounded ReLu }\\
    \end{tabular}%
    }
  \label{tab:SCM_settings}%
\end{table*}%
Table \ref{tab:SCM_settings} shows single-run configurations used in this section where the activations, $T_{max}$, $L_{step}$ and $\tau$ were selected from experimentation for suitable performance. 

\begin{table*}[htbp]
  \centering
  \caption{Benchmark regression results with binary weights}
    \begin{tabular}{p{30pt}|p{70pt}|p{100pt}p{100pt}}
    \hline
    Dataset & Algorithm & Training RMSE & Testing RMSE \bigstrut[b]\\
    \hline
    \hline
    \multirow{6}[2]{*}{R-DB1} & SCN  & \textbf{0.07318$\pm$0.00009} & 0.07419$\pm$0.00093 \bigstrut[t]\\
      & DeepSCN  & 0.07343$\pm$0.00019 & 0.07482$\pm$0.00680 \\
      & SCM  & 0.07477$\pm$0.00054 & \textbf{0.07327$\pm$0.00129} \\
      & IRVFL  & 0.07451$\pm$0.00015 & 0.07537$\pm$0.00150 \\
      & DIRVFL-I  & 0.07442$\pm$0.00023 & 0.07500$\pm$0.00185 \\
      & DIRVFL-II  & 0.07621$\pm$0.00099 & 0.07410$\pm$0.00139 \bigstrut[b]\\
    \hline
    \multirow{6}[2]{*}{R-DB2} & SCN  & 0.09060$\pm$0.00116 & 0.08238$\pm$0.00309 \bigstrut[t]\\
      & DeepSCN  & \textbf{0.05234$\pm$0.00196} & 0.06615$\pm$0.00511 \\
      & SCM  & 0.05544$\pm$0.00477 & \textbf{0.06393$\pm$0.00503} \\
      & IRVFL  & 0.10108$\pm$0.00346 & 0.09566$\pm$0.00557 \\
      & DIRVFL-I  & 0.09829$\pm$0.00479 & 0.10732$\pm$0.00926 \\
      & DIRVFL-II  & 0.09416$\pm$0.00923 & 0.09551$\pm$0.01176 \bigstrut[b]\\
    \hline
    \multirow{6}[2]{*}{R-DB3} & SCN  & 0.05543$\pm$0.00001 & 0.05562$\pm$0.00008 \bigstrut[t]\\
      & DeepSCN  & 0.05115$\pm$0.00021 & 0.05312$\pm$0.00047 \\
      & SCM  & \textbf{0.05060$\pm$0.00085} & \textbf{0.05261$\pm$0.00062} \\
      & IRVFL  & 0.05559$\pm$0.00005 & 0.05567$\pm$0.00013 \\
      & DIRVFL-I  & 0.05415$\pm$0.00025 & 0.05514$\pm$0.00039 \\
      & DIRVFL-II  & 0.05408$\pm$0.00050 & 0.05474$\pm$0.00068 \bigstrut[b]\\
    \hline
    \multirow{6}[2]{*}{R-DB4} & SCN  & 0.07041$\pm$0.00165 & 0.07743$\pm$0.00642 \bigstrut[t]\\
      & DeepSCN  & \textbf{0.05199$\pm$0.00199} & 0.06914$\pm$0.00868 \\
      & SCM  & 0.05402$\pm$0.00404 & \textbf{0.06439$\pm$0.01016} \\
      & IRVFL  & 0.09557$\pm$0.00365 & 0.10365$\pm$0.00927 \\
      & DIRVFL-I  & 0.07107$\pm$0.00341 & 0.08605$\pm$0.01331 \\
      & DIRVFL-II  & 0.07237$\pm$0.00863 & 0.07885$\pm$0.01313 \bigstrut[b]\\
    \hline
    \multirow{6}[2]{*}{R-DB5} & SCN  & 0.00301$\pm$0.00005 & 0.00287$\pm$0.00005 \bigstrut[t]\\
      & DeepSCN  & 0.00283$\pm$0.00012 & 0.00287$\pm$0.00013 \\
      & SCM  & \textbf{0.00208$\pm$0.00008} & \textbf{0.00203$\pm$0.00002} \\
      & IRVFL  & 0.00529$\pm$0.00067 & 0.00503$\pm$0.00077 \\
      & DIRVFL-I  & 0.00453$\pm$0.00086 & 0.00429$\pm$0.00088 \\
      & DIRVFL-II  & 0.00227$\pm$0.00004 & 0.00203$\pm$0.00002 \bigstrut[b]\\
    \hline
    \multirow{6}[2]{*}{R-DB6} & SCN  & 0.10571$\pm$0.00015 & 0.10588$\pm$0.00021 \bigstrut[t]\\
      & DeepSCN  & 0.10644$\pm$0.00018 & 0.10665$\pm$0.00024 \\
      & SCM  & \textbf{0.10339$\pm$0.00027} & \textbf{0.10365$\pm$0.00023} \\
      & IRVFL  & 0.12159$\pm$0.00102 & 0.12071$\pm$0.00088 \\
      & DIRVFL-I  & 0.12172$\pm$0.00114 & 0.12089$\pm$0.00095 \\
      & DIRVFL-II  & 0.10780$\pm$0.00005 & 0.10734$\pm$0.00005 \bigstrut[b]\\
    \hline
    \multirow{6}[2]{*}{R-DB7} & SCN  & 0.00369$\pm$0.00115 & 0.00402$\pm$0.00117 \bigstrut[t]\\
      & DeepSCN  & 0.00009$\pm$0.00005 & 0.00013$\pm$0.00008 \\
      & SCM  & \textbf{0.000009$\pm$0.00000} & \textbf{0.00002$\pm$0.00001} \\
      & IRVFL  & 0.06115$\pm$0.00042 & 0.06025$\pm$0.00059 \\
      & DIRVFL-I  & 0.04047$\pm$0.00368 & 0.04143$\pm$0.00339\\
      & DIRVFL-II  & 0.04465$\pm$0.00713 & 0.04507$\pm$0.00654 \bigstrut[b]\\
    \hline
    \multirow{6}[2]{*}{R-DB8} & SCN  & 0.12574$\pm$0.00002 & 0.12544$\pm$0.00002 \bigstrut[t]\\
      & DeepSCN  & 0.07973$\pm$0.00198 & 0.08313$\pm$0.00211 \\
      & SCM  & \textbf{0.03814$\pm$0.01575} & \textbf{0.04309$\pm$0.01582} \\
      & IRVFL  & 0.12608$\pm$0.00001 & 0.12556$\pm$0.00001 \\
      & DIRVFL-I  & 0.10564$\pm$0.00151 & 0.10806$\pm$0.00170 \\
      & DIRVFL-II  & 0.12128$\pm$0.00161 & 0.12179$\pm$0.00142 \bigstrut[b]\\
    \hline
    \end{tabular}%
  \label{tab:scm_regression}%
\end{table*}%

Table \ref{tab:scm_regression} shows the average RMSE and standard deviations given 100 independent trails. As can be seen for all datasets, SCN, DeepSCN, and SCM perform the best for the training of the network; however, in all cases, SCM outperforms all other methods for testing. This demonstrates SCM's greater ability in generalization over a range of different data. R-DB8's task is to model the Rastrigin function, which is often used to test optimization algorithms \cite{MUHLENBEIN1991619}. It can be seen in the table that SCM outperforms all other methods significantly in solving this complex problem. This is demonstrated further in Figure \ref{fig:Rastrigin_binary}, where SCN and DIRVFL-II do not even resemble the target function, clearly indicating that a deep and supervised approach is needed.\\

\begin{table*}[htbp]
  \centering
  \caption{Benchmark classification results with binary weights}
    \begin{tabular}{c|c|cccc}
    \hline
    Data Set & Algorithm & Training RMSE & Testing RMSE & Training Rate & Testing Rate \bigstrut[b]\\
    \hline
    \hline
    \multirow{6}[2]{*}{C-DB1} & SCN & 0.56897$\pm$0.0027543 & 0.56345$\pm$0.0031679 & 0.88986$\pm$0.0013843 & 0.89477$\pm$0.0026227 \bigstrut[t]\\
      & DeepSCN & 0.55303$\pm$0.004339 & 0.54947$\pm$0.0049877 & 0.87881$\pm$0.0018349 & 0.8814$\pm$0.0032453 \\
      & SCM & \textbf{0.53377$\pm$0.0069777} & \textbf{0.53338$\pm$0.0067238} & \textbf{0.90743$\pm$0.0041165} & \textbf{0.90689$\pm$0.0041004} \\
      & IRVFL & 0.64871$\pm$0.0028071 & 0.6445$\pm$0.0026194 & 0.8395$\pm$0.0026803 & 0.84677$\pm$0.0028682 \\
      & DIRVFL-I & 0.68919$\pm$0.0066998 & 0.68657$\pm$0.0087712 & 0.79282$\pm$0.0077074 & 0.7951$\pm$0.0099802 \\
      & DIRVFL-II & 0.61845$\pm$0.0031813 & 0.61779$\pm$0.0030957 & 0.85518$\pm$0.0017701 & 0.85966$\pm$0.0018861 \bigstrut[b]\\
    \hline
    \multirow{6}[2]{*}{C-DB2} & SCN & 0.59103$\pm$0.0018252 & 0.5933$\pm$0.001582 & 0.80827$\pm$0.0021895 & 0.80635$\pm$0.0051224 \bigstrut[t]\\
      & DeepSCN & 0.59001$\pm$0.0029765 & 0.59451$\pm$0.0033934 & 0.79312$\pm$0.0021397 & 0.78903$\pm$0.0023114 \\
      & SCM & \textbf{0.56687$\pm$0.0019226} & \textbf{0.56962$\pm$0.0019373} & \textbf{0.83297$\pm$0.0018681} & \textbf{0.83133$\pm$0.0021603} \\
      & IRVFL & 0.66076$\pm$0.0018948 & 0.66251$\pm$0.0023176 & 0.75826$\pm$0.0028859 & 0.75808$\pm$0.0049781 \\
      & DIRVFL-I & 0.69435$\pm$0.0050267 & 0.69735$\pm$0.0056538 & 0.71876$\pm$0.0031318 & 0.71572$\pm$0.0064232 \\
      & DIRVFL-II & 0.59338$\pm$0.0017315 & 0.59517$\pm$0.0017267 & 0.82039$\pm$0.00086767 & 0.81983$\pm$0.0014997 \bigstrut[b]\\
    \hline

    \end{tabular}%
  \label{tab:scm_classification}%
\end{table*}%
Further demonstrating SCM's ability to outperform other randomized algorithms, Table \ref{tab:scm_classification} shows the average RMSE and standard deviations given 50 independent trials on C-DB1 and C-DB2. This shows that SCM can outperform not only on regression problems but also large-scale of classification problems.  

\begin{figure} 

  \centering
    \includegraphics[width=\textwidth]{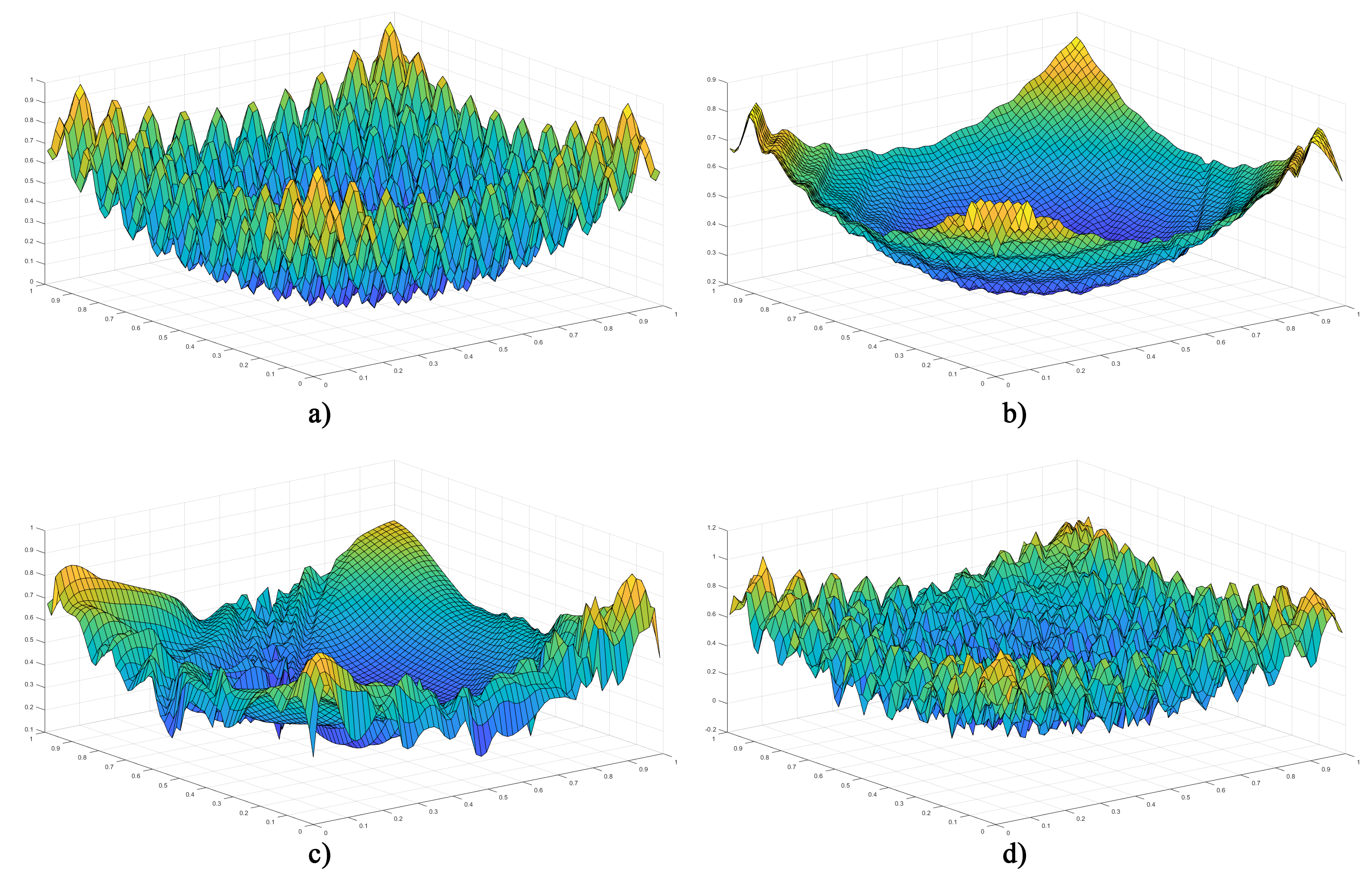}
   \caption{Target (a), and output of SCN(b), DIRVFL-II(c) and SCM(d), using binary weights. SCN stops searching at 58 nodes, DIRVFL results in a 19-32-10-8-10-15-22-11-10-27 network and SCM results in a 17-13-10-472-287-288-54-52-44-9 network. }
   \label{fig:Rastrigin_binary}
\end{figure}



\begin{table*}[htbp]
  \centering
  \caption{Model parameter settings for industrial datasets}
    \begin{tabular}{c|c|r|c|c|c|c|c}
    \hline
    Data Set & Algorithm$^{\mathrm{a}}$ & \multicolumn{1}{c|}{Layers} & $L_{max}$ & Activations$^{\mathrm{b}}$ & $T_{max}$ & $L_{step}$ & $\tau$  \bigstrut[b]\\
    \hline
    \hline
    \multirow{10}[2]{*}{I-DB1} & SCN (b) & 1 & 150 & T & 1200 & - & - \bigstrut[t]\\
      & SCN (r) & 1 & 150 & T & 1200 & - & - \\
      & DeepSCN (b) & 3 & 50,50,50 & T-T-T & 1000,1100,1200 & - & - \\
      & DeepSCN (r) & 3 & 50,50,50 & T-T-T & 1000,1100,1200 & - & - \\
      & SCM (b) & 3 & $\infty$,$\infty$,$\infty$ & T-T-T & 1000,1100,1200 & 10 & 0.00001 \\
      & IRVFL (b) & 1 & 150 & T & - & - & - \\
      & IRVFL (r) & 1 & 150 & T & - & - & - \\
      & DIRVFL-I (b) & 3 & 50,50,50 & T-T-T & - & - & - \\
      & DIRVFL-I (r) & 3 & 50,50,50 & T-T-T & - & - & - \\
      & DIRVFL-II (b) & 3 & $\infty$,$\infty$,$\infty$ & T-T-T & - & 10 & 0.00001 \bigstrut[b]\\
    \hline
    \multirow{10}[2]{*}{I-DB2} & SCN (b) & 1 & 40 & S & 50 & - & - \bigstrut[t]\\
      & SCN (r) & 1 & 40 & S & 50 & - & - \\
      & DeepSCN (b) & 2 & 20,20 & S-S & 30,50 & - & - \\
      & DeepSCN (r) & 2 & 20,20 & S-S & 30,50 & - & - \\
      & SCM (b) & 2 & $\infty$,$\infty$ & S-S & 30,50 & 10 & 0.01 \\
      & IRVFL (b) & 1 & 40 & S & - & - & - \\
      & IRVFL (r) & 1 & 40 & S & - & - & - \\
      & DIRVFL-I (b) & 2 & 20,20 & S-S & - & - & - \\
      & DIRVFL-I (r) & 2 & 20,20 & S-S & - & - & - \\
      & DIRVFL-II (b) & 2 & $\infty$,$\infty$ & S-S & - & 10 & 0.01 \bigstrut[b]\\
    \hline
    \multirow{10}[2]{*}{I-DB3*} & SCN (b) & 1 & 120 & H & 1000 & - & - \bigstrut[t]\\
      & SCN (r) & 1 & 120 & H & 1000 & - & - \\
      & DeepSCN (b) & 3 & 60,60,60 & H-H-H & 1000,1000,1000 & - & - \\
      & DeepSCN (r) & 3 & 60,60,60 & H-H-H & 1000,1000,1000 & - & - \\
      & SCM (b) & 3 & $\infty$,$\infty$ & H-H-H & 1000,1000,1000 & 10 & 0.00001 \\
      & IRVFL (b) & 1 & 120 & H & - & - & - \\
      & IRVFL (r) & 1 & 120 & H & - & - & - \\
      & DIRVFL-I (b) & 3 & 60,60,60 & H-H-H & - & - & - \\
      & DIRVFL-I (r) & 3 & 60,60,60 & H-H-H & - & - & - \\
      & DIRVFL-II (b) & 3 & $\infty$,$\infty$ & H-H-H & - & 10 & 0.00001 \bigstrut[b]\\
    \hline
    \multicolumn{4}{p{190pt}}{$^{\mathrm{a}}$ b = binary weights, r = real weights }\\
    \multicolumn{4}{p{210pt}}{$^{\mathrm{b}}$ S = Sigmoid, T = Tanh, H = Hard Limiter}\\
    \multicolumn{4}{p{250pt}}{* I-DB3 uses a industrial mechanism model for $P(X)$}\\
    \end{tabular}%
  \label{tab:SCM_industrial_settings}%
\end{table*}%

\subsection{Results for Industrial datasets}
Table \ref{tab:SCM_industrial_settings} shows the parameter settings selected for both industrial datasets based on experimentation for best performance. Notably, I-DB3 uses an industrial mechanism model based on hot-rolling parameters for $P(X)$ rather than LASSO regression, as used in the other datasets.
\begin{table}[h]
  \centering
  \caption{Performance comparison for industrial datasets}
    \begin{tabular}{p{30pt}|p{90pt}|p{90pt}p{90pt}}
    \hline
    \multicolumn{1}{l|}{Data Set} & Algorithm & Training RMSE & Testing RMSE \bigstrut[b]\\
    \hline
    \hline
    \multirow{10}[2]{*}{I-DB1} & SCN (binary) & 0.00321$\pm$0.00007 & 0.00395$\pm$0.00011 \bigstrut[t]\\
      & SCN (real) & 0.00268$\pm$0.00005 & 0.00327$\pm$0.00016 \\
      & DeepSCN (binary) &  0.00335$\pm$0.00013 & 0.00435$\pm$0.00027 \\
      & DeepSCN (real) & 0.00269$\pm$0.00009 & 0.00335$\pm$0.00016 \\
      & SCM (binary) & \textbf{0.00205$\pm$0.00021} & \textbf{0.00269$\pm$0.00022} \\
      & IRVFL (binary) & 0.00932$\pm$0.00089 & 0.00966$\pm$0.00093 \\
      & IRVFL (real) & 0.00579$\pm$0.00038 & 0.00605$\pm$0.00053 \\
      & DIRVFL-I (binary) & 0.01839$\pm$0.00254 & 0.01849$\pm$0.00265 \\
      & DIRVFL-I (real) & 0.01108$\pm$0.00190 & 0.01165$\pm$0.00195 \\
      & DIRVFL-II (binary) & 0.00498$\pm$0.00076 & 0.00511$\pm$0.00075 \bigstrut[b]\\
    \hline
    \multirow{10}[2]{*}{I-DB2} & SCN (binary) & 0.00793$\pm$0.00073 & 0.00771$\pm$0.00060 \bigstrut[t]\\
      & SCN (real) & 0.00690$\pm$0.00114 &  0.00678$\pm$0.00089 \\
      & DeepSCN (binary) & 0.00696$\pm$0.00078 & 0.00704$\pm$0.00085 \\
      & DeepSCN (real) & 0.00694$\pm$0.00110 & 0.00680$\pm$0.00097 \\
      & SCM (binary) & \textbf{0.00513$\pm$0.00060} & \textbf{0.00517$\pm$0.00032} \\
      & IRVFL (binary) & 0.01179$\pm$0.00183 & 0.01087$\pm$0.00199 \\
      & IRVFL (real) & 0.00953$\pm$0.00127 & 0.00855$\pm$0.00084 \\
      & DIRVFL-I (binary) & 0.01741$\pm$0.00330 & 0.01634$\pm$0.00393 \\
      & DIRVFL-I (real) & 0.01056$\pm$0.00194 & 0.00981$\pm$0.00163 \\
      & DIRVFL-II (binary) & 0.01083$\pm$0.00827 & 0.00207$\pm$0.00108 \bigstrut[b]\\
    \hline
    \multirow{10}[2]{*}{I-DB3} & SCN (binary) & 0.04907$\pm$0.00099 & 0.05791$\pm$0.00364 \bigstrut[t]\\
      & SCN (real) & 0.04634$\pm$0.00088 & 0.05603$\pm$0.00268 \\
      & DeepSCN (binary) & 0.04767$\pm$0.00078 & 0.06233$\pm$0.00275 \\
      & DeepSCN (real) & 0.04550$\pm$0.00111 & 0.06113$\pm$0.00281 \\
      & SCM (binary) & \textbf{0.02097$\pm$0.00045} & \textbf{0.02159$\pm$0.00047} \\
      & IRVFL (binary) & 0.08508$\pm$0.00322 & 0.08505$\pm$0.00368 \\
      & IRVFL (real) & 0.08414$\pm$0.00371 & 0.08286$\pm$0.00421 \\
      & DIRVFL-I (binary) & 0.08928$\pm$0.00482 & 0.09163$\pm$0.00606 \\
      & DIRVFL-I (real) & 0.09121$\pm$0.00527 & 0.09368$\pm$0.00604 \\
      & DIRVFL-II (binary) & 0.02333$\pm$0.00058 & 0.02293$\pm$0.00069 \bigstrut[b]\\
    \hline
    
    \end{tabular}%
  \label{tab:train_test_industrial}%
\end{table}%
Table \ref{tab:train_test_industrial} shows the average RMSE and standard deviations given 100 independent trails for both training and testing on the industrial dataset.
It can be observed that using real weights for SCN, DeepSCN, IRVFL, and DIRVFL does yield better performance results than using binary weights, which is expected given the greater expressive capabilities of real numbers. However, even given this, SCM with binary weights outperforms all other models, including those using real weights, demonstrating improved memory usage and performance. Further, for I-DB3 it is clear that an industrial mechanism model is helpful to improve the performance.\\

Figures \ref{fig:Servo_Testing} and \ref{fig:RC_Testing} show the RMSE value as the nodes are iteratively added to each model with a single run for I-DB1 and I-DB2, respectively. This clearly shows that SCM outperforms all other models at very few nodes, as the error decreases significantly and more rapidly. This demonstrates that if SCM needs to be further optimized by reducing the number of nodes, it can be achieved with improved performance. Further, in Figure \ref{fig:Servo_Testing}, we can observe some overfitting occurring with DIRVFL-I(binary), showing that RVFL models are more unreliable.\\

Figure \ref{fig:Servo_Output} shows the output values for SCM compared to the target and using only the linear model for the servo motor system dataset. 
The number of nodes depends on the early stopping parameters and the random values generated for the weights and biases; this approach results in variations in network size, as shown in Figure \ref{fig:Nodes_Per_Layer} where the nodes per layer are shown for 100 trails of I-DB2. This approach shows that the proposed learning algorithm can adapt and produce the most suitable network, given the random candidates.\\

\begin{table}[htbp]
  \centering
  \caption{Training execution time of hot rolling thickness dataset (I-DB1)}
    \begin{tabular}{c|c|p{13.43em}}
    \hline
    \multicolumn{1}{l|}{Data Set} & Algorithm & Training Execution Time (s) \bigstrut[b]\\
    \hline
    \hline
    \multirow{10}[2]{*}{I-DB1} & SCN (binary) & 24.4238$\pm$0.458007 \bigstrut[t]\\
      & SCN (real) & 25.9111$\pm$0.353106 \\
      & DeepSCN (binary) & 20.8576$\pm$0.38939 \\
      & DeepSCN (real) & 23.3514$\pm$0.372389 \\
      & SCM (binary) & 32.6294$\pm$7.325919 \\
      & IRVFL (binary) & 2.3975$\pm$0.070491 \\
      & IRVFL (real) & 2.3804$\pm$0.139402 \\
      & DIRVFL-I (binary) & 2.6282$\pm$0.227988 \\
      & DIRVFL-I (real) & 2.6052$\pm$0.120244 \\
      & DIIRVFL-II (binary) & 4.2571$\pm$2.874096 \bigstrut[b]\\
    \hline

    \end{tabular}%
  \label{tab:scm_execution_time}%
\end{table}%
\begin{table}[htbp]
  \centering
  \caption{Breakdown of execution time on servo motor data-set (I-DB2)}
    \begin{tabular}{l|l}
    \hline
    Algorithm Part & Time (s) \bigstrut[b]\\
    \hline
    \hline
    Testing & 2.20503$\pm$0.695 \bigstrut[t]\\
    Training & 2122.54$\pm$711.671 \\
    Lasso & 52.9417$\pm$5.39965 \\
    Candidate Search & 1146.3$\pm$315.698 \\
    Inequality Equation & 235.914$\pm$63.6531 \\
    Upgrade & 789.842$\pm$431.703 \\
    \hline
    \end{tabular}%
  \label{tab:servo_time}%
\end{table}%

\begin{table}[htbp]
  \centering
  \caption{SCM, SNN, DNN, DT performance and execution time on I-DB2}
    \begin{tabular}{l|ccc}
    \hline
    Algorithm & Training (RMSE) & Test (RMSE) & Training Time (s) \bigstrut[b]\\
    \hline
    \hline
    SCM & \textbf{0.005128} & \textbf{0.005173} & 2123 \bigstrut[t]\\
    SNN & 0.058206 & 0.058244 & 21396 \\
    DNN & 0.058309 & 0.058378 & 13263 \\
    DT & 0.005902 & 0.006467 & \textbf{1080} \\
    \hline
    \end{tabular}%
  \label{tab:SCM_SNN_DNN_DT}%
\end{table}%
The execution time of the hot rolling thickness dataset for training each model is reported in Table \ref{tab:scm_execution_time}. All implementations of IRVFL run significantly quicker, as they do not use a supervisory mechanism; in turn, the time saving is not significant, given the poor performance results. SCM does perform the slowest due to the extra features, but the time difference is not significant, and it can be observed that SCM has the highest standard deviation. \\

Table \ref{tab:servo_time} shows the breakdown of execution time testing and training SCM averaged from 100 trials. Whilst there is a lot of deviation between the values due to the variation of nodes in each model, it is clear that whilst testing, the model can perform in a reasonable amount of time. Testing takes only two seconds on average to do a complete pass of all testing samples, which means that, on average, a single sample takes around $1.28\mu s$ to complete a forward pass. During training, it is clear that the candidate search, on average, takes the longest due to the amount of iteration; however, this can be adjusted by adjusting the number of candidates for training if needed to reduce this time.\\

To investigate SCMs ability compared to existing algorithms, a single layer neural network (SNN), a deep neural network (DNN) and a decision tree (DT) are also tested on the I-DB2 dataset. Both the SNN and DNN both were created in Matlab using the Levenbergh-Marquardt \cite{more1978levenberg} training algorithm, tanh activation function for all hidden and output nodes and a maximum of 100 epochs. The SNN has 40 hidden nodes and the DNN consists of two layers with 20 nodes per layer. The DT is built using the Python 'CatBoost' \cite{DBLP:journals/corr/DorogushGGKPV17} library, which uses gradient boosting to build a decision tree. The performance and execution time of each algorithm are shown in Table \ref{tab:SCM_SNN_DNN_DT}. This table shows that SNN and DNN are unsuitable with similar-sized networks, as the performance and time are not practical. The DT performs worse than SCM but is quicker to train. It is clear that SCM is the preferred choice even with the longer training time, as the testing performance is significantly improved.\\

\begin{table*}[htbp]

  \centering
  \caption{Data reduction using binary weights in SCM}
    \begin{tabular}{p{45pt}|p{22pt}p{26pt}p{44pt}p{30pt}p{40pt}p{40pt}p{48pt}}
    \hline
    Network & Inputs & Outputs & Number of Weights & $\Upsilon^*$ Bits & {$w^*$ Bits (Binary)} & {$w^*$ Bits (Real$^{\mathrm{*}}$)} & {Model Size Reduction} \bigstrut[b]\\
    \hline
    \hline
    117-24-31 & 36 & 1 & 7764 & 11008 & 7764 & 496896 & 96.22\% \bigstrut[t]\\
    28 - 8 & 11 & 1 & 532 & 2304 & 532 & 34048 & 91.67\% \\
    \hline
    \multicolumn{4}{p{190pt}}{$^{\mathrm{*}}$ Real values are 64 bit floating point}
    \end{tabular}%
  \label{tab:data_reduction}%
\end{table*}%
From Table \ref{tab:data_reduction}, we can observe the amount of memory saved in the model using binary weights. Two example models using the two different industrial datasets both see a greater than 90\% reduction in data for storing the weights than using real values. 



\begin{figure} 

  \centering
    \includegraphics[width=0.75\textwidth]{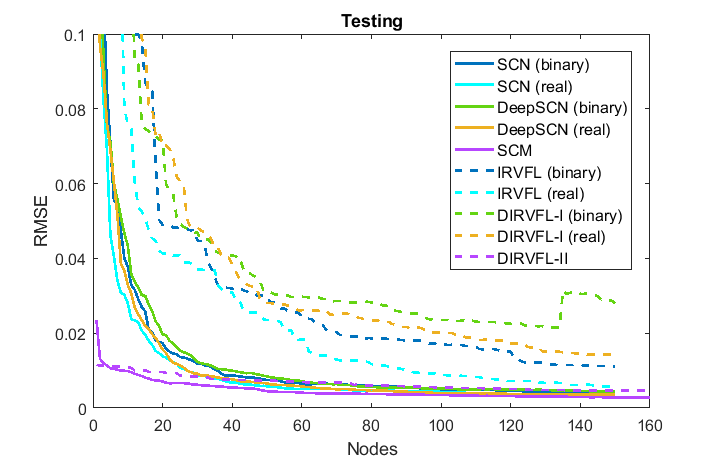}
   \caption{Testing RSME for each iteration on the hot rolling thickness dataset (I-DB1) }
   \label{fig:Servo_Testing}
\end{figure}

\begin{figure} 

  \centering
    \includegraphics[width=0.75\textwidth]{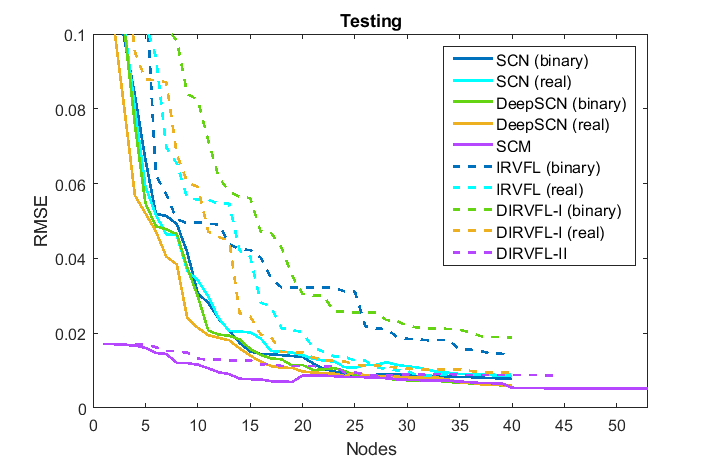}
   \caption{Testing RSME for each iteration on the servo motor dataset (I-DB2) }
   \label{fig:RC_Testing}
\end{figure}

\begin{figure} 

  \centering
    \includegraphics[width=1\textwidth]{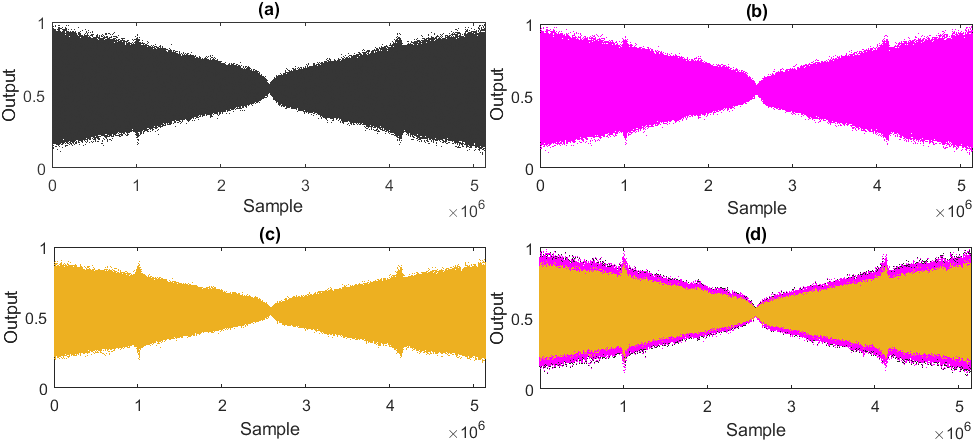}
   \caption{Training output of servo motor system dataset (I-DB2), where a) is the target, b) is the SCM output, c) is using only the linear model and d) shows the outputs combined}
   \label{fig:Servo_Output}
\end{figure}

\begin{figure*} 

  \centering
    \includegraphics[width=1\textwidth]{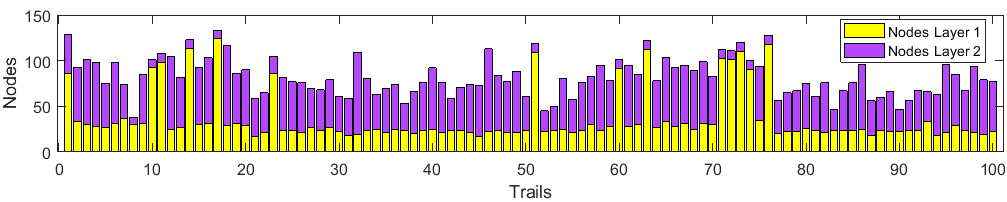}
   \caption{Nodes per layer using early stopping on servo motor system dataset (I-DB2)}
   \label{fig:Nodes_Per_Layer}
\end{figure*}

\begin{table}[htbp]
  \centering
  \caption{FPGA weight memory reduction using SCM}
    \begin{tabular}{p{38px}p{38px}p{38px}p{38px}p{38px}}
    \hline
    Weights (Real) & Weights (FPGA) & Bits (Real) & Bits (FPGA) & Memory Reduction \bigstrut\\
    \hline
    \hline
    660 & 16500 & 42240 & 16500 & 60.94\% \bigstrut\\
    \hline
    \end{tabular}%
  \label{tab:FPGA 2}%
\end{table}%

\begin{table}[htbp]
  \centering
  \caption{FPGA RMSE, time and power results using SCM}
    \begin{tabular}{p{35px}p{35px}p{35px}p{35px}p{35px}p{35px}p{35px}}
    \hline
    FPGA RMSE & PC RMSE &  Power & Clock cycles & Nodes & Activation & Time \bigstrut\\
    \hline
    \hline
    0.031360 & 0.031336 & 0.295W & 9 & 60 & SIGN & 90 ns \bigstrut\\
    \hline
    \end{tabular}%
  \label{tab:FPGA 1}%
\end{table}%

\subsection{Discussion}
This section gives some reasons why SCM outperforms other implementations of randomized models, including the implementations of RVFL, SCN and  DeepSCN. Notice that the universal approximation property can not be guaranteed for RVFL implementations. Hence, even if RVFL can perform reasonably on some datasets, this is not true for all datasets, making it unreliable and impractical, especially for complex and large industrial datasets. The results from DeepSCN and SCM demonstrate that deep models may be more effective and efficient in data modelling. Without the supervisory mechanism proposed in \cite{8013920}, we cannot expect much improvement on the modelling performance of RVFL models. \\

SCM is built on SCN and DeepSCN, utilising the supervisory mechanism which is necessary for a universal approximator. 
The key technical contributions and improvements that SCM provides can be summarised as follows:
\begin{itemize}
    \item The proposed early stopping feature, which is evaluated at training time,  has an advantage over most implementations of building a MLP model, in that a model can be built quicker, as it is not reliant on validation data to find a suitable number of nodes. This also has an advantage over SCN and DeepSCN in that $T_{max}$ need not be considered, saving time and effort in finding suitable configurations. 
    \item The proposed addition of the mechanism model P helps in speeding up learning and improving the accuracy as demonstrated in the results. This additional information passed to the model opens a door to understanding the predictive results from SCM models. Further, this unique addition can be extended to find suitable weights that are model-dependent such as those used in industry.
    \item SCN and RVFL both can use binary weights, as demonstrated in the results, and it is shown that SCN and RVFL generally do perform better using real weights. However, SCM demonstrates that whilst using binary weights, it outperforms all other methods. This clearly indicates that a low-memory model can be built and be accurate.
    \item SCM can replace commonly used algorithms such as SNN, DNN and DT with improved performance.
\end{itemize}
The proposed SCM algorithm can also be applied to FPGAs to exploit the high-speed hardware offered. The implementation used in this case study involves both binary weights and inputs, where the inputs are encoded using an encoding scheme. The outputs and mechanism weights are encoded using fixed-point notation. Moreover, the activation is limited to sign or step. The details of the implementation and encoding schemes are outlined in chapter 8; however, here, we demonstrate the power, accuracy, speed and memory saving using a hardware solution. In this case study, we use a XC7A100T-1CSG324C FPGA with a 100MHz clock. The dataset used is a subset of the servo motor dataset I-DB2 with 8568 samples for training and 2856 samples for testing. SCM is used to train the model on a PC and then the model is implemented on an FPGA and tested.\\

Table \ref{tab:FPGA 1} demonstrates that the RMSE of the PC model and FPGA model are very similar and that the FPGA model performs only slightly worse, in turn showing that SCM can be applied to FPGAs with very little loss in accuracy. The implementation takes nine clock cycles for a single input to be evaluated and for a 100MHz clock, this is 90ns, which is significantly faster than what could be achieved on a PC. The average power based on Vivados power analysis tool is 0.295W. Table \ref{tab:FPGA 2} shows the difference in bit memory using real 64-bit weights versus binary weights on the FPGA. A $60.94\%$ reduction in weight memory is reported, demonstrating that SCM on an FPGA does indeed enable compact models suitable for industrial application to be developed.    

\section{Conclusion}
In this paper, we present and mathematically formulate the model complexity (MC) concept in machine learning. With the help of such a significant concept, we show that a learner model has no power to approximately represent continuous signals and its derivative simultaneously if the difference of MC degrees between the model and the data is less than zero. As stated in remarks and conjectures, we have a long way to go in exploring models, conditions and the associations between MC and the zero-order universal approximation property. Anyway, this innovative concept greatly helps in understanding the capacity of a learner, guiding the development of learning algorithms and the hyper-parameter settings in deep learning. Various characterisations and properties of MC, and computing methods of MC for discrete cases can be further explored. 
Also, a reasonable and logical extension of MC concept nondifferentiable class of activation functions is being expected.\\

The proposed SCM model shows promise in resource-limited industrial applications. It uses less memory than other implementations of SCN whilst still outperforming in terms of accuracy. This memory saving can lead to many developments in automation, consumer goods, automotive industry and robotics. The addition of a mechanism model opens up a vast area to explore different industry models that can be applied to SCM. The benchmark dataset simulations show that SCM can be applied to both classification and regression problems. Further, SCM has demonstrated its applicability for large-scale industrial datasets. A hardware SCM can be implemented using an FPGA, allowing very fast and low-power models to be created, we will report more details in another paper. SCM provides evidence that low-memory SCN can be realised and can be the basis of further development, such as extending to hardware implementation, online training and low-level optimisations. SCM, as a new class of randomized learner models, can be regarded as a cornerstone of data modelling tools for industrial artificial intelligence. As an advanced lightweight deep learning model, SCM has great potential and value in edge-computing of the industrial internet.

\section*{Acknowledgment}
We are grateful to Professor Xu Li from Northeastern University, China, for sharing two industrial datasets (I-DB1 and I-DB3). Professor Dianhui Wang expresses thanks to his previous team members for the second industrial data collection in Nanjing, where the model complexity concept was created in his night dream at Cuiping International Complex, 2019.\\

\bibliographystyle{IEEEtran}
\bibliography{IEEEabrv,SCM}



\end{document}